\documentclass[10pt,conference]{IEEEtran}
\makeatletter
 \def\blfootnote{\xdef\@thefnmark{}\@footnotetext}
 \makeatother
\usepackage[inner=0.675in,outer=0.675in,top=0.7in,bottom=1in,pdftex]{geometry}
\usepackage{amsthm}
\usepackage{dsfont}
\usepackage[justification=centering]{caption}
\usepackage[usenames]{color}
\usepackage[noadjust]{cite}
\ifCLASSINFOpdf
   \usepackage[pdftex]{graphicx}
   \usepackage{epstopdf}
\else
  \usepackage[dvips]{graphicx}
\fi
\usepackage{float}
\usepackage[cmex10]{amsmath}
\usepackage[]{amssymb}
\usepackage{algorithm}
\usepackage{algpseudocode}
\usepackage{mdwmath}
\usepackage{mdwtab}
\usepackage{subcaption}
\usepackage{url}
\usepackage{color}
\usepackage[verbose]{wrapfig}
\usepackage{hyperref}
\usepackage{bbding}
\usepackage{pifont}

\allowdisplaybreaks
\newtheorem{Theorem}{Theorem}

\newtheorem{Lemma}{Lemma}
\newtheorem{Remark}{Remark}

\newtheorem{Assumption}{Assumption}
\newtheorem{Example}{Example}
\newtheorem*{Example*}{Example}

\begin{document}
\title{Communication Efficient Federated Learning with Energy Awareness over Wireless Networks}

\author{
    \IEEEauthorblockN{Richeng Jin, \textit{Student Member, IEEE}, Xiaofan He, \textit{Member, IEEE}, Huaiyu Dai, \textit{Fellow, IEEE}}
}

\maketitle
\begin{abstract}
In federated learning (FL), reducing the communication overhead is one of the most critical challenges since the parameter server and the mobile devices share the training parameters over wireless links. With such consideration, we adopt the idea of SignSGD in which only the signs of the gradients are exchanged. Moreover, most of the existing works assume Channel State Information (CSI) available at both the mobile devices and the parameter server, and thus the mobile devices can adopt fixed transmission rates dictated by the channel capacity. In this work, only the parameter server side CSI is assumed, and channel capacity with outage is considered. In this case, an essential problem for the mobile devices is to select appropriate local processing and communication parameters (including the transmission rates) to achieve a desired balance between the overall learning performance and their energy consumption. Two optimization problems are formulated and solved, which optimize the learning performance given the energy consumption requirement, and vice versa. Furthermore, considering that the data may be distributed across the mobile devices in a highly uneven fashion in FL, a stochastic sign-based algorithm is proposed. Extensive simulations are performed to demonstrate the effectiveness of the proposed methods.
\end{abstract}

{\blfootnote{R. Jin and H. Dai are with the Department of Electrical and Computer Engineering, North Carolina State University, Raleigh, NC, USA, 27695 (e-mail: \{rjin2, hdai\}@ncsu.edu).
X. He is with the Electronic Information School, Wuhan University, China (e-mail: xiaofanhe@whu.edu.cn).
}}

\section{Introduction}
\noindent To train a machine learning model, traditionally a centralized approach is adopted in which the training data are aggregated on a single machine. On the one hand, such a centralized training approach is privacy-intrusive, especially when the data are collected by mobile devices and contain the owners' sensitive information (e.g., locations, user preference on websites, social media, etc.). On the other hand, transmitting all the collected data from mobile devices may be impractical due to communication resource limitations. With such consideration, the concept of federated learning (FL), which enables training on a large corpus of decentralized data residing on mobile devices, is proposed in \cite{konevcny2016federated}.

As a distributed training approach, FL adopts the parameter server paradigm in which most of the computation is offloaded to the mobile devices in a parallel manner. During each communication round, after receiving the learning model parameters from the server, the workers (i.e., mobile devices) train their local learning models using their local data and transmit the parameter updates back to the server, which will aggregate the information from all the workers and start the next round by broadcasting the updated model parameters. In wireless networks, since all the communications between the workers and the server are over wireless links, the learning performance depends on the wireless environments as well as the workers' communication resource and energy constraints. There have been some works that study the communication aspects of FL \cite{tran2019federated,dinh2020federated,yang2020delay,zeng2020energy,yang2020energy,shi2020device,shi2020joint,chen2020convergence,wadu2020federated,vu2020cell,ren2020accelerating,chen2020joint,amiri2020update,du2020high,zheng2020design,chang2020communication,zhu2020broad,zhu2020toward,yang2020federated,amiri2020machine,amiri2020federated,zhu2020one,hosseinalipour2020federated,hosseinalipour2020multi}. {\color{black}Nonetheless, they either do not consider the existing strategies that have shown promising improvement in communication efficiency (e.g., gradient quantization \cite{bernstein2018signsgd1}) or ignore the energy consumption of the workers, the impact of transmission errors, and data heterogeneity. In addition, most of these works assume channel-state information (CSI) at both the server side and the workers side. As a result, the workers are assumed to adopt fixed transmission rates {\color{black}dictated} by the channel capacity. In this work, the idea of SignSGD with majority vote \cite{bernstein2018signsgd1} is adopted to improve the communication efficiency of the FL algorithm, in which only the signs of the parameter updates are shared between the server and the workers. The workers are assumed to transmit their parameter updates over flat-fading channels and CSI is only available at the {\color{black}server} side. Channel capacity with outage is considered and each worker is supposed to determine its transmission rate and transmission power.}

It is worth mentioning that in real-world FL applications over wireless networks, the communication time between the server and the workers is not negligible. Therefore, it becomes more critical to improve the learning performance with respect to the total training time instead of the number of communication rounds. With such consideration, the implementation of the FL algorithms given a fixed total training time is considered in this work. In such a case, the learning performance depends on the number of communication rounds that the FL algorithm can be run and the outage probabilities of the workers for each communication round. Intuitively, increasing the transmission power and decreasing the transmission rate of a worker both decrease its outage probability. However, increasing the transmission power results in higher energy consumption for communication while decreasing the transmission rate requires faster local computation (i.e., training the local FL model) given fixed time duration for each communication round, which leads to higher energy consumption for local computation. Considering that mobile devices usually have limited batteries, it is essential to minimize their energy consumption by appropriately configuring the local computation and communication parameters while satisfying the learning performance requirement (or the other way around). More specifically, our main contributions are summarized as follows.
\begin{table*}
\color{black}
\caption{Summary of the Current State of the Art}
\vspace{-0.1in}
\label{CurrentStateoftheArt}
\begin{center}
\begin{tabular}{ | c | c | c|c|c|}
\hline
{} & Multiple Access Method &Compressed Model Updates? & Client Energy Consumption? &Impact of Communication Errors? \\
\hline
\cite{zhu2020broad,zhu2020toward,yang2020federated} & Over-the-air computation & \ding{55} & \ding{55} & {\color{black}analog communication with high SNR}\\
\hline
\cite{zhu2020one,amiri2020machine,amiri2020federated} & Over-the-air computation & \ding{52} & \ding{55} & {\color{black}analog communication with high SNR}\\
\hline
\cite{yang2020delay,shi2020device,shi2020joint,chen2020convergence,wadu2020federated,ren2020accelerating} & TDMA/FDMA & \ding{55} & \ding{55} & \ding{55}\\
\hline
\cite{tran2019federated,dinh2020federated,zeng2020energy,yang2020energy,vu2020cell} & TDMA/FDMA & \ding{55} & \ding{52} & \ding{55}\\
\hline
\cite{amiri2020update,du2020high,zheng2020design,chang2020communication} & TDMA/FDMA & \ding{52} & \ding{55} & \ding{55}\\
\hline
 \cite{chen2020joint} & TDMA/FDMA & \ding{55} & \ding{52} & \ding{52}\\
\hline
This work  & TDMA/FDMA & \ding{52} & \ding{52} & \ding{52}\\
\hline
\end{tabular}
\end{center}
\vspace{-0.1in}
\end{table*}
\begin{itemize}
\item {\color{black}The impact of wireless communications on SignSGD, in particular, the probability of correct aggregation for each communication round, is analyzed. Combining our analysis on the {\color{black}performance} improvement of each communication round and the convergence rate which captures the speed that the improvement accumulates over communication rounds, a new metric is proposed to {\color{black}more accurately} measure the learning performance of SignSGD {\color{black}in practice}.}
\item {\color{black}Two optimization problems are formulated and solved. The first problem optimizes the learning performance given the energy consumption {\color{black}constraint}, while the second problem minimizes the energy consumption of the workers given the learning performance requirement.
\item The scenario with heterogeneous data distribution across the workers is considered and a stochastic sign based algorithm that can deal with {\color{black}data heterogeneity} is proposed.
\item Extensive simulations are performed to demonstrate the effectiveness of the proposed method. Particularly, compared with SignSGD and FedAVG \cite{mcmahan2017communication}, the proposed stochastic sign based algorithm achieves better learning performance while reducing the energy consumption of the workers.}
\end{itemize}
{\color{black}The remainder of this work is organized as follows. Section~\ref{Related Works} discusses related works. Section \ref{SystemModel} introduces the system model and the performance analysis of SignSGD over wireless networks. The optimization problems are formulated in Section \ref{ProblemFormulation} and the corresponding solutions are presented in Section \ref{Solutions}. Section \ref{SectionHeterogeneous} extends the proposed method to the scenario with heterogeneous data distribution across the workers. Section \ref{Simulations} presents the simulation results. Conclusions are presented in Section \ref{Conclusion}.}


\section{Related Works}\label{Related Works}
\noindent To improve the communication efficiency of the distributed learning algorithms, various methods have been proposed, including quantization \cite{alistarh2017qsgd,bernstein2018signsgd1,wu2018error,wen2017terngrad,agarwal2018cpsgd}, sparsification \cite{sattler2019sparse,sattler2019robust,wang2018atomo} and subsampling \cite{konevcny2016federated,caldas2018expanding}. However, most of these works ignore the impact of wireless environments and the resource constraints of the mobile devices, which are of vital importance in the implementation of FL algorithms over real-world wireless networks.

Recently, there have been a number of existing works that study the communication aspects of FL algorithms. In \cite{tran2019federated,dinh2020federated}, the weighted sum of the training time and the energy consumption is optimized by properly selecting the local computation and the communication parameters. {\color{black}\cite{yang2020delay} considers the minimization of the training time by properly configuring the local computation and communication parameters and the bandwidth allocation.} The energy consumption of the communications between the mobile devices and the server is also considered in \cite{zeng2020energy} and the goal is to minimize the weighted sum of the energy consumption and the number of participated mobile devices by mobile device scheduling and effective bandwidth allocation. {\color{black}\cite{yang2020energy} considers the minimization of the total energy  consumption of all the workers under a latency constraint. \cite{shi2020device} considers a joint mobile device scheduling and bandwidth allocation problem to minimize the expected FL training time, while \cite{shi2020joint} proposes a joint device scheduling and resource allocation policy to maximize the model accuracy within a given total training time budget for latency constrained wireless FL.} To further reduce the FL convergence time, \cite{chen2020convergence} incorporates artificial nueral networks (ANNs) to estimate the local FL models of the devices that are not scheduled for transmission. \cite{wadu2020federated} proposes a joint device scheduling and resource block allocation policy for FL under imperfect CSI. \cite{vu2020cell} considers a cell-free massive MIMO scenario and the training time is minimized by jointly optimizing the local computation and the communication parameters. \cite{ren2020accelerating} empirically proposes a learning efficiency metric which is a function of the mini-batch size and the time of each communication round. Resource allocation and the mini-batch size are jointly optimized to maximize the learning efficiency. \cite{chen2020joint} takes the effect of packet transmission errors into consideration and analyzes its impact on the performance of FL. A joint bandwidth allocation and mobile device selection problem is formulated and solved to minimize a FL loss function that captures the performance of the FL algorithm. {\color{black}However, in these works, the aforementioned effective strategies for improving communication efficiency are not considered. In this work, we adopt the idea of SignSGD with majority vote \cite{bernstein2018signsgd1}, which improves the communication efficiency by a {\color{black}factor} of 32.

{\color{black}To reduce the communication overhead, \cite{amiri2020update,du2020high,zheng2020design,chang2020communication} incorporate the quantization, sparsification and error compensation schemes. However, these works do not consider the energy consumption and communication errors. Another line of works \cite{zhu2020broad,zhu2020toward,yang2020federated} adopt the over-the-air computation scheme for model updates aggregation to {\color{black}improve} the communication {\color{black}efficiency}. \cite{amiri2020machine,amiri2020federated,zhu2020one} further combine model updates compression schemes with over-the-air aggregation. While over-the-air computation mechanisms demonstrate advantages in terms of bandwidth efficiency, they usually require tight synchronization and CSI information at the workers side for power equalization. Besides, none of these works consider the energy consumption of the workers. A summary of the main differences between this work and the current state of the art can be found in Table \ref{CurrentStateoftheArt}.}

}

\section{System Model}\label{SystemModel}
\noindent In this work, a wireless multi-user system consisting of one parameter server and a set of $M$ workers is considered. In particular, each worker $m \in \mathcal{M}$ stores a local dataset $\mathcal{D}_{m}$, which will be used for local training. The local dataset can be locally generated or collected through each worker's usage of mobile devices. Considering that the training of a prediction model, especially in deep learning, usually requires a large dataset, the goal of the workers is to cooperatively learn a machine learning model while keeping the local training data on their mobile devices.

\subsection{Machine Learning Model}
\noindent A typical federated learning problem with $M$ normal workers is considered. Formally, the goal is to minimize a finite-sum objective of the form
\begin{equation}
\min_{{\color{black}\boldsymbol{w}}\in \mathbb{R}^d}F({\color{black}\boldsymbol{w}})~~~~ \text{where}~~~~ F({\color{black}\boldsymbol{w}}) \overset{\mathrm{def}}{=} \frac{1}{M}\sum_{m=1}^{M}F_{m}({\color{black}\boldsymbol{w}}).
\end{equation}
For a machine learning problem, we have a sample space $I = X \times Y$, where $X$ is a space of feature vectors and $Y$ is a label space. Given the hypothesis space $\mathcal{W} \subseteq \mathbb{R}^{d}$, we define a loss function $l: \mathcal{W}\times I \rightarrow \mathbb{R}$ which measures the loss of prediction on the data point $({\color{black}\boldsymbol{x}},y) \in I$ made with the hypothesis vector ${\color{black}\boldsymbol{w}} \in \mathcal{W}$. In such a case, $F_{m}({\color{black}\boldsymbol{w}})$ is a local function defined by the local dataset of worker $m$ and the hypothesis $w$. More specifically,
\begin{equation}
F_{m}({\color{black}\boldsymbol{w}})=\frac{1}{|\mathcal{D}_{m}|}\sum_{(x_n,y_n)\in \mathcal{D}_{m}}l(w;({\color{black}\boldsymbol{x}}_n,y_n)),
\end{equation}
where $|\mathcal{D}_{m}|$ is the size of worker $m$'s local dataset $\mathcal{D}_{m}$. The loss function $l({\color{black}\boldsymbol{w}};({\color{black}\boldsymbol{x}}_n,y_n))$ depends on the learning tasks and the machine learning models. {\color{black}In this work, $F_{m}({\color{black}\boldsymbol{w}})$'s are not necessarily convex.}

To accommodate the requirement of communication efficiency in FL, we adopt the popular idea of gradient quantization as in SignSGD with majority vote \cite{bernstein2018signsgd1}, which is presented in Algorithm \ref{SIGNSGD}. At the ${\color{black}k}$-th communication round, each worker $m$ computes the gradient ${\color{black}\boldsymbol{g}}^{({\color{black}k})}_{m}$ based on its locally stored model weights ${\color{black}\boldsymbol{w}}^{({\color{black}k})}$ and the local datasets $\mathcal{D}_{m}$.\footnote{\color{black}In this work, it is assumed that each worker evaluates the gradients over its whole local dataset for simplicity (i.e., $\boldsymbol{g}_{m}^{({\color{black}k})} = \nabla F_{m}({\color{black}\boldsymbol{w}}^{({\color{black}k})}), \forall 1 \leq m \leq M$). {\color{black}The analysis and the proposed method can also be applied in the scenario where each worker computes its local gradients over a mini-batch of its local dataset.}} Then, instead of transmitting the gradient ${\color{black}\boldsymbol{g}}^{({\color{black}k})}_{m}$ directly, worker $m$ transmits $sign({\color{black}\boldsymbol{g}}^{({\color{black}k})}_{m})$ to the parameter server, in which $sign(\cdot)$ is the sign function. After receiving the shared signs of the gradients from the workers (prone to channel errors), the parameter server performs aggregation using the majority vote rule and sends the aggregated result back to the workers. Finally, the workers update their local model weights using the aggregated result.

\begin{algorithm}[!t]
\caption{SignSGD with majority vote \cite{bernstein2018signsgd1} over wireless networks}
\label{SIGNSGD}
\begin{algorithmic}[1]
\State Input: initial weight: ${\color{black}\boldsymbol{w}}^{(0)}$; number of workers: $M$; learning rate: $\eta$.
\For{$k=0,1,\cdots,K$}
    \State {\color{black}Each worker $m$ transmits $sign({\color{black}\boldsymbol{g}}^{({\color{black}k})}_{m})$ to the parameter server over wireless links, where ${\color{black}\boldsymbol{g}}^{({\color{black}k})}_{m} = \nabla F_{m}({\color{black}\boldsymbol{w}}^{({\color{black}k})})$ is the local gradient.}
    \State The parameter server obtains a noisy estimate (denoted by ${\color{black}\hat{\boldsymbol{g}}}^{({\color{black}k})}_{m}$) of the transmitted information $sign({\color{black}\boldsymbol{g}}^{({\color{black}k})}_{m})$ from each worker $m$ and sends the aggregated result ${\color{black}\tilde{\boldsymbol{g}}}^{({\color{black}k})} = sign\big(\sum_{m=1}^{M}{\color{black}\hat{\boldsymbol{g}}}^{({\color{black}k})}_{m}\big)$ back to the workers.
    \State The workers update their local models
\begin{equation}
{\color{black}\boldsymbol{w}}^{({\color{black}k+1})} = {\color{black}\boldsymbol{w}}^{({\color{black}k})} - \eta{\color{black}\tilde{\boldsymbol{g}}}^{({\color{black}k})}.
\end{equation}
\EndFor
\end{algorithmic}
\end{algorithm}

\subsection{Local Computation Model}
\noindent In this work, we consider a similar local computation model as those in \cite{tran2019federated} and \cite{chen2020joint}. Let $c_{m}$ and $f_m$ denote the number of CPU cycles required for worker $m$ to process per bit data and its CPU cycle frequency, respectively, which are assumed known to the parameter server. {\color{black}We note that the frequency essentially measures the computation speed of a CPU. In practice, the users usually adjust the CPU frequency of the devices to reduce energy consumption and prevent overheating.} Then, the CPU energy consumption of worker $m$ for the local computation of one communication round is given by \cite{burd1996processor}
\begin{equation}\label{LocalComputationEnergy}
\begin{split}
E_{m}^{cmp} = \frac{\alpha_m}{2}c_{m}D_{m}f_{m}^2,
\end{split}
\end{equation}
in which $\frac{\alpha_{m}}{2}$ is the effective capacitance coefficient of worker $m$'s computing chip, $D_{m}$ is the size of worker $m$'s training data for each communication round (in bits). In addition, the computation time for each communication round of worker $m$ is given by
\begin{equation}\label{LocalComputationTime}
T_{m}^{cmp} = \frac{c_{m}D_{m}}{f_{m}}.
\end{equation}

\subsection{Transmission Model}
\noindent In this work, it is assumed that the workers transmit their local updates (i.e., the signs of the gradients) to the parameter server via the orthogonal frequency division multiple access (OFDMA), and does not interfere with each other. Given that the parameter server has more power and bandwidth compared to the mobile devices, the downlink transmission time is ignored in this work.\footnote{Note that given a fixed transmission rate for the parameter server, the downlink transmission time is a constant that can be readily integrated to the first and the second constraints of the optimization problems (\ref{OP5}) and (\ref{OP1}), respectively, if needed.} Moreover, similar to most of the existing literature (e.g., \cite{tran2019federated,chen2020joint}), we assume that the downlink transmissions are error-free for simplicity.

For the uplink transmission, different from the existing works that consider CSI at both the transmitter and the receiver sides, we consider flat-fading channels with receiver only CSI and the capacity with outage. Capacity with outage is defined as the maximum rate that can be transmitted over a channel with a certain outage probability, which corresponds to the probability that an outage happens and the transmission cannot be decoded correctly \cite{goldsmith2005wireless}. For each worker $m$, we assume a discrete-time channel with stationary and ergodic time-varying normalized gain $\sqrt{h_{m}}$ following Rayleigh distribution, and additive white Gaussian noise (AWGN).{\color{black}\footnote{\color{black}Note that we assume Rayleigh fading for simplicity, and the study can be easily generalized to other channel models.}} Suppose that worker $m$ transmits at a rate of $r_{m}=\log_{2}(1+\gamma_{min})$, in which $\gamma_{min}$ is some fixed minimum received SNR, the data can be correctly received if the instantaneous received SNR $\gamma_{m} = \frac{P_{m}h_{m}}{N_{0}B_{m}}$ is greater than or equal to $\gamma_{min}$, in which $P_{m}$ is the transmission power of worker $m$; $N_{0}$ is the noise power spectral density and $B_{m}$ is the corresponding bandwidth. The probability of outage is thus $p_{out} = P(\gamma_{m} < \gamma_{min})$. Particularly, for Rayleigh fading channel, we have
\begin{equation}\label{p_out}
p_{out}(r_m) = 1 - e^{-\frac{(2^{r_m}-1)N_{0}B_{m}}{P_m}}.
\end{equation}

The corresponding communication time and energy consumption are given by
\begin{equation}\label{EnergyCommunication}
T_{m}^{com} = \frac{s_m}{r_{m}B_{m}}, ~~E_{m}^{com} = \frac{P_{m}s_{m}}{r_{m}B_{m}},
\end{equation}
in which $s_{m}$ is the size of the transmitted data (in bits).\footnote{Note that in the schemes where full precision gradients are transmitted, each worker is supposed to transmit $32$ bits for each element in the gradient vectors. Therefore, Algorithm \ref{SIGNSGD} leads to a 32-fold improvement in communication time and communication energy consumption.} 

For simplicity, the wireless link between each worker $m$ and the parameter server for each entry of the transmitted gradients is assumed to be a binary symmetric channel with crossover probability $p_{out}(r_{m})$. In this sense, we have
\begin{equation}\label{worst-case-outage}
\color{black}
{\color{black}\hat{\boldsymbol{g}}}^{({\color{black}k})}_{m} =
\begin{cases}
\hfill -sign({\color{black}\boldsymbol{g}}^{({\color{black}k})}_{m}),&\text{with probability $p_{out}(r_{m})$,}\\
\hfill sign({\color{black}\boldsymbol{g}}^{({\color{black}k})}_{m}), &\text{with probability $1-p_{out}(r_{m})$.}
\end{cases}
\end{equation}

\begin{Remark}
\color{black}
In (\ref{worst-case-outage}), it is assumed that for each worker $m$, $sign({\color{black}\boldsymbol{g}}^{({\color{black}k})}_{m})$ is transmitted as a single packet in the uplink and all entries of $sign({\color{black}\boldsymbol{g}}^{({\color{black}k})}_{m})$ are incorrectly decoded when an outage happens. This is considered to be the worst case scenario. One may also assume that the parameter server can detect the outage and discard the corresponding packets. In this case,
\begin{equation}\label{discard-outage}
\color{black}
{\color{black}\hat{\boldsymbol{g}}}^{({\color{black}k})}_{m} =
\begin{cases}
\hfill 0,&\text{with probability $p_{out}(r_{m})$,}\\
\hfill sign({\color{black}\boldsymbol{g}}^{({\color{black}k})}_{m}), &\text{with probability $1-p_{out}(r_{m})$.}
\end{cases}
\end{equation}

Besides, the extension to the scenarios where partial bits of the packet may be recovered is also straightforward, and won't change the fundamental tradeoffs revealed in this study. In the following performance analysis of the learning algorithms, we show the lower bound of the learning performance in the worst case scenario. That being said, our analysis on learning performance is conservative, and the derived results can be applied in the other aforementioned scenarios as well.
\end{Remark}

\color{black}
\begin{figure}
\centering
\centerline{\includegraphics[width=0.4\textwidth]{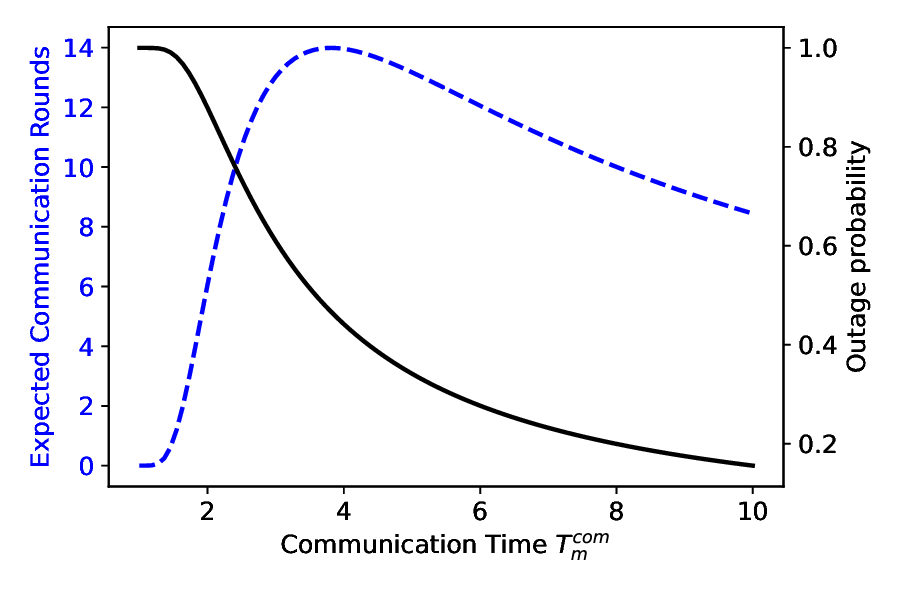}}
\caption{The blue curve shows the expected successful communication rounds for worker $m$ per 100 seconds. The black curve shows the outage probability of worker $m$. We set $s_{m} = 10^{6}$bits, $N_{0} = 10^{-8}$W/Hz; $P_{m}=0.005$W; $B_{m} = 180$kHz.}
\label{outageprobabilityroundtradeoff}
\end{figure}
According to (\ref{p_out}) and (\ref{EnergyCommunication}), {\color{black}given a total communication time, there exists a tradeoff between the number of communication rounds ${\color{black}K}\propto 1/T_{m}^{com}$ and the outage probability $p_{out}(r_{m})$ for each communication round. Particularly, the communication time $T_{m}^{com}$ for each communication round should be large enough to guarantee reliable transmission. Meanwhile, a large $T_{m}^{com}$ results in fewer communication rounds in total.} In FL, during each communication round, the parameter server sends the current global model to the workers and waits for the workers to report updates. If enough workers report within a given communication round, the server will update its global model and ignore the updates from the workers that fail to report in time \cite{bonawitz2019towards}. That being said, different from many existing applications in wireless communications, the goal of FL is not to reliably receive the information from all the workers. Instead, the parameter server intends to maximize the number of successful global model updates. Fig. \ref{outageprobabilityroundtradeoff} shows the outage probability and the expected communication rounds that worker $m$ is not in outage with respect to different communication time $T^{com}_{m}$, given a total communication time of 100 seconds.\footnote{We note that the choice of the total communication time does not change the fundamental tradeoff revealed in Fig. \ref{outageprobabilityroundtradeoff}.} It can be observed that the expected successful communication rounds achieves the highest when $T_{m}^{com} = 3.82$s and the correspond outage probability is around 46.6\%. This indicates that if the negative impact of unsuccessful communication is small (e.g., the packets in outage can be detected and discarded), the parameter server can tolerate a relatively large outage probability while achieving good learning performance.

\color{black}
\subsection{\color{black}Performance Analysis of Algorithm \ref{SIGNSGD} over Wireless Networks}\label{AnalysisOfImpact}
\noindent Before diving into the details of the system design, we first analyze how wireless communications affect the performance of Algorithm \ref{SIGNSGD}. To facilitate the analysis, the following assumptions are made.
\begin{Assumption}\label{A2}(Smoothness)
$\forall {\color{black}\boldsymbol{w}}_1,{\color{black}\boldsymbol{w}}_2$, we require for some non-negative constant $L$
\begin{equation}
F({\color{black}\boldsymbol{w}}_1) \leq F({\color{black}\boldsymbol{w}}_2) + <\nabla F({\color{black}\boldsymbol{w}}_2), {\color{black}\boldsymbol{w}}_1-{\color{black}\boldsymbol{w}}_2> + \frac{L}{2}||{\color{black}\boldsymbol{w}}_1 - {\color{black}\boldsymbol{w}}_2||^{2}_2,
\end{equation}
where $<\cdot,\cdot>$ is the standard inner product.
\end{Assumption}
{\color{black}We note that the Assumption \ref{A2} is commonly adopted in the literature (e.g., \cite{bernstein2018signsgd1}), given which the following result can be proved.}

\begin{Theorem}\label{T1}
Suppose that the model parameter at the beginning of ${\color{black}k}$-th communication round is ${\color{black}\boldsymbol{w}}^{({\color{black}k})}$, then by performing one communication round of Algorithm \ref{SIGNSGD}, we have
\begin{equation}\label{objective}
\begin{split}
&\mathbb{E}[F({\color{black}\boldsymbol{w}}^{({\color{black}k})})-F({\color{black}\boldsymbol{w}}^{({\color{black}k+1})})] \geq  -\eta||\nabla F({\color{black}\boldsymbol{w}}^{({\color{black}k})})||_{1} - \frac{L\eta^2d}{2} + 2\eta \times \\
&\sum_{i=1}^{d}|\nabla F({\color{black}\boldsymbol{w}}^{({\color{black}k})})_{i}|P\big({\color{black}\tilde{\boldsymbol{g}}}^{({\color{black}k})}_{i} = sign(\nabla F({\color{black}\boldsymbol{w}}^{({\color{black}k})}))_{i}\big),
\end{split}
\end{equation}
in which $d$ is the dimension of the gradients;  ${\color{black}\tilde{\boldsymbol{g}}}^{({\color{black}k})}_{i}$ and $\nabla F({\color{black}\boldsymbol{w}}^{({\color{black}k})})_{i}$ are the $i$-th entry of the aggregated result ${\color{black}\tilde{\boldsymbol{g}}}^{({\color{black}k})}$ and the gradient vector $\nabla F({\color{black}\boldsymbol{w}}^{({\color{black}k})})$, respectively. $sign(\cdot)_{i}$ is the $i$-th entry of the vector after taking the sign operation. The expectation and the probability are over the dynamics of the wireless channels.
\end{Theorem}
\begin{proof}
Please see Appendix \ref{ProofOfT1}.
\end{proof}

\begin{Remark}\label{Remark1_1}
Theorem \ref{T1} lower bounds the expected improvement of the learning objective during the ${\color{black}k}$-th communication round (i.e., $\mathbb{E}[F({\color{black}\boldsymbol{w}}^{({\color{black}k})})-F({\color{black}\boldsymbol{w}}^{({\color{black}k+1})})]$). Intuitively, the learning performance depends on two quantities: (1) the improvement of the learning objective during each communication round; (2) the number of communication rounds. When the data are homogeneously distributed across the workers, SignSGD converges with a rate of $O(1/{\color{black}\sqrt{K}})$ \cite{bernstein2018signsgd1}, in which ${\color{black}K}$ is the number of communication rounds.
\end{Remark}

{\color{black}Note that given ${\color{black}\boldsymbol{w}}^{({\color{black}k})}$, $F({\color{black}\boldsymbol{w}}^{({\color{black}k})})$ and $\nabla F({\color{black}\boldsymbol{w}}^{({\color{black}k})})$ are constants. Therefore, maximizing the lower bound of the expected improvement of the learning objective (i.e., the right-hand side of (\ref{objective})) is equivalent to maximizing the probabilities of correct aggregation $P({\color{black}\tilde{\boldsymbol{g}}}^{({\color{black}k})}_{i} = sign(\nabla F({\color{black}\boldsymbol{w}}^{({\color{black}k})})))_{i}, 1 \leq i \leq d$. For the ease of discussion, we consider the $i$-th entry of the gradient}{\color{black}\footnote{\color{black}Note that the following analysis is universal for any $1\leq i\leq d$.}} and define a series of random variables $\{X_{m,i}\}_{m=1}^{M}$ given by
\begin{equation}\label{DefinitionofX_m}
X_{m,i} =
\begin{cases}
1, \hfill ~~~\text{if $sign({\color{black}\hat{\boldsymbol{g}}_{m}^{({\color{black}k})}})_{i} \neq sign(\nabla F({\color{black}\boldsymbol{w}}^{({\color{black}k})}))_{i}$}, \\
0, \hfill ~~~\text{if $sign({\color{black}\hat{\boldsymbol{g}}_{m}^{({\color{black}k})}})_{i}  = sign(\nabla F({\color{black}\boldsymbol{w}}^{({\color{black}k})}))_{i}$}.
\end{cases}
\end{equation}
$X_{m,i}$ can be considered as the outcome of one Bernoulli trial ({\color{black} which is unknown to the parameter server}) {\color{black}indicating whether the $m$th worker shares the wrong sign (due to data incompleteness or channel errors) with the server.} Let $Z_{i} = \sum_{m=1}^{M}X_{m,i}$, then it can be verified that\footnote{Note that the scenario in which $F({\color{black}\boldsymbol{w}}^{({\color{black}k})})_{i}=0$ is not considered in our study for simplicity.}
\begin{equation}\label{PcorrectAggregation}
P\big({\color{black}\tilde{\boldsymbol{g}}}^{({\color{black}k})}_{i} = sign(\nabla F({\color{black}\boldsymbol{w}}^{({\color{black}k})}))_{i}\big) = P\bigg(Z_{i} < \frac{M}{2}\bigg).
\end{equation}
In addition, $Z_{i}$ follows the Poisson binomial distribution with mean $\mathbb{E}[Z_{i}] = \sum_{m=1}^{M}P(X_{m,i})$. Since $Z_{i}$ is non-negative, the Markov's inequality {\color{black}indicates}
\begin{equation}
P\big(Z_{i} \geq \frac{M}{2}\big) \leq \frac{2\mathbb{E}[Z_{i}]}{M},
\end{equation}
and therefore
\begin{equation}\label{LowerBound}
\begin{split}
P\big({\color{black}\tilde{\boldsymbol{g}}}^{({\color{black}k})}_{i} = sign(\nabla F({\color{black}\boldsymbol{w}}^{({\color{black}k})}))_{i}\big)&= 1-P\big(Z_{i} \geq \frac{M}{2}\big) \\
&\geq \frac{M - 2\mathbb{E}[Z_{i}]}{M}.
\end{split}
\end{equation}
Note that $\mathbb{E}[Z_{i}]$ and $M-\mathbb{E}[Z_{i}]$ are the expected number of workers that share wrong and correct signs, respectively. The lower bound in (\ref{LowerBound}) represents the difference between the ratios of workers that share the correct signs {\color{black}(i.e., $(M-\mathbb{E}[Z_{i}])/M$)} and that share the wrong signs {\color{black}(i.e., $\mathbb{E}[Z_{i}]/M$)}.
\begin{Remark}\label{Remark_3}
Note that the closed form of the expected improvement of the learning objective during the ${\color{black}k}$-th communication round (i.e., $\mathbb{E}[F({\color{black}\boldsymbol{w}}^{({\color{black}k})})-F({\color{black}\boldsymbol{w}}^{({\color{black}k+1})})]$) and the probabilities of correct aggregation $P({\color{black}\tilde{\boldsymbol{g}}}^{({\color{black}k})}_{i} = sign(\nabla F({\color{black}\boldsymbol{w}}^{({\color{black}k})}))_{i}), 1 \leq i \leq d$ are difficult to obtain, especially when the objective function $F(\cdot)$ is unknown. Therefore, the bound derived in (\ref{LowerBound}) is used to measure the expected improvement during each communication round. In this sense, in order to optimize the learning performance, we need to: (1) maximize $(M - 2\mathbb{E}[Z_{i}])/M$; (2) increase the total number of communication rounds given a fixed total training time (until convergence).\footnote{{\color{black}Theoretically, it can be shown that $\mathbb{E}[\frac{1}{T}\sum_{t=1}^{T}||\nabla F({\color{black}\boldsymbol{w}}^{({\color{black}k})})||_{1}] \leq O(1/{\color{black}\sqrt{K}})$.} In practice, it is usually not expected that the gradients be reduced to 0. The machine learning algorithms are said to converge when the performance stops improving {\color{black}subject to a given threshold}, which usually takes a finite number of communication rounds.}
\end{Remark}
Given any $M$, maximizing the right-hand side of (\ref{LowerBound}) is equivalent to minimizing $\mathbb{E}[Z_{i}] = \sum_{m=1}^{M}P(X_{m,i}=1)$. Let $p_{m,i}^{({\color{black}k})}$ denote the probability of $sign({\color{black}\boldsymbol{g}}^{({\color{black}k})}_{m})_{i} = sign(\nabla F({\color{black}\boldsymbol{w}}^{({\color{black}k})}))_{i}$ (i.e., the $i$-th entry of the local gradient of worker $m$ has the same sign as that of the true gradient $\nabla F({\color{black}\boldsymbol{w}}^{({\color{black}k})})$), it can be shown that
\begin{equation}\label{PX_M}
P(X_{m,i}=1) = p_{m,i}^{({\color{black}k})}p_{out}(r_{m}) + (1-p_{m,i}^{({\color{black}k})})(1-p_{out}(r_{m})).
\end{equation}
When $p_{m,i}^{({\color{black}k})} > 0.5$,\footnote{\color{black}We note that $p_{m,i}^{({\color{black}k})}$'s depend on the distribution of the local training datasets. As a result, it is impossible to theoretically obtain $p_{m,i}^{({\color{black}k})}$'s. As a special case, when all the workers have the same dataset, ${\color{black}\boldsymbol{g}}_{m}^{({\color{black}k})} = \nabla F({\color{black}\boldsymbol{w}}^{({\color{black}k})})$ and therefore $p_{m,i}^{({\color{black}k})} = 1, \forall m,t,i$. In the homogeneous data distribution setting, ${\color{black}\boldsymbol{g}}_{m}^{({\color{black}k})}$ can be considered as a noisy version of $\nabla F({\color{black}\boldsymbol{w}}^{({\color{black}k})})$. Following the assumption in \cite{bernstein2018signsgd1} that the noise is symmetric with zero mean, $p_{m,i}^{({\color{black}k})} > 0.5$ always holds. {\color{black}In addition, we note that $p_{m,i}^{({\color{black}k})} > 0.5$ still holds if each worker computes its gradients over a mini-batch of its local dataset instead of the whole local dataset.}}  {\color{black} $P(X_{m,i}=1)$ is minimized if $p_{out}(r_{m})$ is minimized}.

\section{Problem Formulation}\label{ProblemFormulation}
\noindent {\color{black}In this section, the scenario with homogeneous data distribution across the workers is considered.} According to our discussion in Section \ref{AnalysisOfImpact}, to optimize the learning performance, it is desired to minimize the outage probabilities of the workers and maximize the number of communication rounds. In this work, the implementation of the FL algorithm given a fixed total training time is considered. In this case, the number of communication rounds {\color{black}${\color{black}K}$} is inversely proportional to the time duration for each communication round {\color{black}i.e., $T_{l} \triangleq T_{m}^{cmp}+T_{m}^{com}$}. Considering that the workers (i.e., the mobile devices) have limited batteries, two optimization problems are formulated. {\color{black}The first optimization problem is of more interests to the parameter server, whose goal is to optimize the learning performance without consuming excessive energy for the workers. The second optimization problem addresses the needs of battery-constrained workers while satisfying the learning performance requirement (which may be dictated by the server).}

{\color{black}In the first optimization problem, the learning performance is optimized given the energy consumption constraint for the workers. As we discussed in Remark \ref{Remark_3}, in order to optimize the learning performance, we need to maximize $(M - 2\mathbb{E}[Z_{i}])/M$ and increase the total number of communication rounds given a fixed total training time. Considering that SignSGD converges with a rate of $O(1/{\color{black}\sqrt{K}})$, the metric ${\color{black}\sqrt{K}}(M-2\mathbb{E}[Z_{i}])/M$ is proposed to measure the learning performance. Given a fixed total training time, ${\color{black}K}$ is inversely proportional to the time duration for each communication round $T_{l}$. {\color{black}Essentially, $(M - 2\mathbb{E}[Z_{i}])/M$ captures the improvement of the learning objective at each communication round (c.f. (\ref{LowerBound})) and the impact of communication outages, and $\sqrt{K}$ (and therefore $1/\sqrt{T_{l}}$) captures the impact of the number of communication rounds. As a result, the product of these two factors characterizes the overall performance improvement. In particular, to obtain the convergence rate of $O(1/\sqrt{K})$, the learning rate in (\ref{objective}) is set as $\eta \propto 1/\sqrt{K}$ \cite{bernstein2018signsgd1}. Accumulating it over $K$ communication rounds yields the factor $\sqrt{K}$.}

In the second optimization problem, the energy consumption is minimized given the learning performance constraint (i.e., the outage probabilities of the workers and the time duration for each communication round). It can be seen from (\ref{p_out}) that given fixed bandwidth $B_{m}$ and noise power spectral density $N_{0}$, the transmission rate $r_{m}$ and the transmission power $P_{m}$ determine the outage probability of worker $m$. Increasing the transmission power $P_{m}$ and decreasing the transmission rate $r_{m}$ both decrease the outage probability. However, according to (\ref{EnergyCommunication}), a larger $P_{m}$ and a smaller $r_{m}$ result in higher communication energy consumption. In addition, given a fixed time duration for each communication round (i.e., $T_{m}^{cmp}+T_{m}^{com}$), decreasing $r_{m}$ increases the communication time $T_{m}^{com}$ and therefore requires worker $m$ to increasing the CPU frequency $f_{m}$ such that the local computation time can be reduced. As a result, the local computation energy consumption of worker $m$ also increases. By solving the second optimization problem, each worker $m$ minimizes its energy consumption by selecting appropriate local computation parameter $f_{m}$, communication parameters $P_{m}$ and $r_{m}$ while satisfying the learning performance constraint.
}

\subsection{Learning Performance Optimization Given Energy Consumption Constraint}
\noindent In the first optimization problem, it is assumed that each worker $m$ first reports its local processing CPU frequency $f_{m}$, transmission power $P_{m}$ and energy consumption limit $E_{m}$ to the parameter server. Then, the parameter server determines the time duration for each communication round $T_{l}$ and the transmission rate $r_{m}$ for each worker. According to the discussion in Section \ref{AnalysisOfImpact}, $\mathbb{E}[Z_{i}]=\sum_{m=1}^{M} p_{m,i}^{({\color{black}k})}p_{out}(r_{m}) + (1-p_{m,i}^{({\color{black}k})})(1-p_{out}(r_{m}))$, in which $p_{m,i}^{({\color{black}k})}$ is determined by the local dataset of worker $m$ and therefore unknown to the server. {\color{black}Moreover, in the homogeneous data distribution scenario, the local datasets (and therefore the local gradients) of the workers share the same distribution. In this case, $p_{m,i}^{({\color{black}k})}$'s are {\color{black}supposed} to be the same across the workers, and minimizing $\sum_{m=1}^{M}p_{out}(r_m)$ also minimizes $\mathbb{E}[Z_{i}]$}. The optimization problem is formulated as follows.
\begin{equation}\label{OP1}
\begin{aligned}
&~~~~~ \underset{T_l, r_{m}}{\text{max}}\frac{M - 2\sum_{m=1}^{M}p_{out}(r_m)}{\sqrt{T_l}} \\
& \text{s.t.}~~ \frac{\alpha_m}{2}c_{m}D_{m}f_{m}^2 + \frac{P_{m}s_{m}}{r_{m}B_{m}} \leq E_{m}, \forall m, \\
& ~~~~~~~ \max_{m}\bigg\{\frac{c_{m}D_{m}}{f_{m}} + \frac{s_m}{r_{m}B_{m}}\bigg\} \leq T_l,
\end{aligned}
\end{equation}
in which the time duration for each communication round $T_{l}$ and the transmission rate $r_{m}$ are the parameters to be optimized. $E_{m}$ is the energy consumption upper limit for worker $m$. The first constraint captures the energy consumption requirement for each worker $m$ and the second constraint captures the time duration requirement for each communication round.
Furthermore, we assume that the workers transmit with high SNR and therefore we have\footnote{\color{black}Note that this assumption is made so that (\ref{OP1}) can be solved more efficiently (see Section \ref{SolutionsA}). Its effectiveness is verified through our simulations.}
\begin{equation}\label{poutageappro}
p_{out}(r_m) \approx \frac{(2^{r_m}-1)N_{0}B_{m}}{P_m}.
\end{equation}

\begin{Remark}
\color{black}
According to (\ref{poutageappro}), given fixed transmission power $P_{m}$, $p_{out}(r_{m})$ is determined by $r_{m}$. In this sense, by solving the first optimization problem, the parameter server can obtain and send the learning performance requirements $T_{l}$ and $p_{out}(r_{m})$ to each worker $m$.
\end{Remark}


\subsection{Energy Minimization Given Learning Performance Constraint}
\noindent 
In this subsection, the energy minimization problem given the requirements for the time duration for each communication round and the outage probability of each worker is considered. Given a constraint $p_{out, m}$ on the outage probability and a constraint $T_l$ on the time duration for each communication round, the goal of worker $m$ is to minimize its energy consumption. The corresponding optimization problem is formulated as follows.
\begin{equation}\label{OP5}
\begin{aligned}
&~ \underset{f_{m}, r_{m}, P_{m}}{\text{min}}\frac{\alpha_m}{2}c_{m}D_{m}f_{m}^2 + \frac{P_{m}s_{m}}{r_{m}B_{m}} \\
& \text{s.t.}~~ \frac{c_{m}D_{m}}{f_{m}} + \frac{s_m}{r_{m}B_{m}} \leq T_l, \\
& ~~~~~ 1 - e^{-\frac{(2^{r_m}-1)N_{0}B_{m}}{P_m}} \leq p_{out,m},\\
& ~~~~~ f_{min,m} \leq f_{m} \leq f_{max,m},\\
&~~~~~ P_{min,m} \leq P_m \leq P_{max,m},
\end{aligned}
\end{equation}
The CPU frequency for local computation $f_{m}$, the transmission rate $r_{m}$ and the transmission power $P_{m}$ are the parameters to be optimized. The feasible regions of CPU frequency and transmission power of worker $m$ are imposed by the third and the fourth constraints, respectively. Considering that the time duration for each communication round is determined by the slowest worker (the straggler), $T_{l}$ is set the same for all the workers.

\begin{Remark}
\color{black}
Given a fixed total training time, the time duration requirement imposes a lower bound on the number of communication rounds, while the outage probability requirement imposes a lower bound on $(M - 2\mathbb{E}[Z_{i}])/M$. Considering that the learning performance improves as the number of communication rounds and $(M - 2\mathbb{E}[Z_{i}])/M$ increase, $T_{l}$ and $p_{out,m}$ specify the worst learning performance that the system will achieve in the considered scenario. The selection of $T_{l}$ and $p_{out,m}$ in this case will be left as our future work.
\end{Remark}



\section{{\color{black}System Parameters Configuration for SignSGD Over Wireless Networks}}\label{Solutions}

\subsection{Learning Performance Optimization Given Energy Consumption Constraint}\label{SolutionsA}
\begin{Lemma}\label{Lemma2}
In the optimization problem (\ref{OP1}), given any fixed $T_l$, the optimal transmission rate of worker $m$ is given by
\small
\begin{equation}\label{OptimalTransmissionRate}
r^{*}_{m} = \max\bigg\{\frac{P_{m}s_{m}}{B_{m}(E_{m}-\frac{\alpha_m}{2}c_{m}D_{m}f_{m}^2)}, \frac{s_{m}f_{m}}{B_{m}f_{m}T_{l}-B_{m}c_{m}D_{m}}\bigg\}.
\end{equation}
\normalsize
\end{Lemma}
\begin{proof}
Please see Appendix \ref{proofOfL2}.
\end{proof}

\begin{Remark}
\color{black}
Note that (17) is infeasible when $\frac{\alpha_m}{2}c_{m}D_{m}f_{m}^2 \geq E_{m}$. In this case, the energy consumption of local computation for worker $m$ exceeds its energy consumption limit. As a result, worker $m$ will not be able to participate in the learning process.
\end{Remark}

Let $\mathcal{U} = \big\{m|\frac{P_{m}s_{m}}{B_{m}(E_{m}-\frac{\alpha_m}{2}c_{m}D_{m}f_{m}^2)} \geq \frac{s_{m}f_{m}}{B_{m}f_{m}T_{l}-B_{m}c_{m}D_{m}}\big\}$. According to Lemma \ref{Lemma2}, the workers can be divided into two groups. The optimal transmission rates of the workers in the first group (i.e., $\mathcal{U}$) is limited by their energy consumption upper limit $E_{m}$ (i.e., further decreasing the transmission rate results in the violation of the energy consumption constraint), while those of the workers in the second group is limited by the time duration for each communication round $T_{l}$ (i.e., further decreasing the transmission rate results in the violation of the time duration requirement), which is subject to design.
Further define the following two functions:
\begin{equation}
\begin{split}
g(x) &= \frac{2\sum_{m \in \mathcal{U}}\big(2^{\frac{P_{m}s_{m}}{B_{m}(E_{m}-\frac{\alpha_m}{2}c_{m}D_{m}f_{m}^2)}}-1\big)N_{0}B_{m}}{P_m\sqrt{x}} \\
&+ \frac{2\sum_{m \notin \mathcal{U}}\big(2^{\frac{s_{m}f_{m}}{B_{m}f_{m}x-B_{m}c_{m}D_{m}}}-1\big)N_{0}B_{m}}{P_m\sqrt{x}},
\end{split}
\end{equation}
\begin{equation}
h(x) = \frac{M}{\sqrt{x}}.
\end{equation}

Based on Lemma \ref{Lemma2}, the optimization problem (\ref{OP1}) can be reformulated as follows.
\begin{equation}\label{OP2}
\begin{aligned}
&\underset{T_l}{\text{min}}~~g(T_l) - h(T_l)\\
&\text{s.t.}~~ T_l \geq \max_{m}\bigg\{\frac{c_{m}D_{m}}{f_{m}}\bigg\}.
\end{aligned}
\end{equation}

It can be verified that both $g(x)$ and $h(x)$ are convex functions of $x$. Therefore, (\ref{OP2}) is a difference of convex programming problem, which can be solved by the DCA algorithm \cite{tao1997convex}.

\subsection{Energy Minimization Given Outage Probability Constraint}\label{EnergyMinimizationGivenOurage}
\noindent We note that the optimization problem (\ref{OP5}) is not always feasible. In particular, according to the time duration requirement $T_l$, it is required that $r_{m} \geq \frac{s_{m}}{(T_l-\frac{c_{m}D_{m}}{f_{max,m}})B_{m}}$. Combining it with the power constraint and plugging them into (\ref{p_out}) yields
\begin{equation}\label{Eq17}
p_{out}(r_{m}) \geq 1 - e^{-\frac{\big(2^{\frac{s_{m}}{(T_l-\frac{c_{m}D_{m}}{f_{max,m}})B_{m}}}-1\big)N_{0}B_{m}}{P_{max,m}}},
\end{equation}
which may violate the given outage constraint $p_{out, m}$ at least for some worker $m$. Therefore, two scenarios are considered.
\subsubsection{The optimization problem (\ref{OP5}) is infeasible}
In this case, {\color{black}the corresponding parameters for worker $m$ are set as} $P_{m} = P_{max,m}$, $f_{m} = f_{max,m}$ and $r_{m} = \frac{s_{m}}{(T_l-\frac{c_{m}D_{m}}{f_{max,m}})B_{m}}$.
\begin{Remark}
We note that $T_l$ and $p_{out}(r_{m})$ are the two most important parameters that determine the performance of the FL algorithm. When the optimization problem (\ref{OP5}) is infeasible {\color{black}for worker $m$}, i.e., the outage probability given in (\ref{Eq17}) exceeds $p_{out, m}$, the delay requirement and the outage probability requirement cannot be satisfied simultaneously. Since the time duration for each communication round is supposed to be determined by the slowest worker (the straggler), we assume that {\color{black}worker $m$} accommodates the time duration requirement while reducing the outage probability as best it can.
\end{Remark}

\subsubsection{The optimization problem (\ref{OP5}) is feasible}
\noindent {\color{black}(\ref{OP5}) requires a joint optimization of $f_{m}$, $r_{m}$ and $P_{m}$, in which the second condition on the outage probability is not necessarily convex. In the following, we explore the inherent structure of (\ref{OP5}) and reduce it to a convex optimization problem over single variable. Particularly, let $r_{m}^{(1)} = \log_2\big(-\frac{P_{min,m}\ln(1-p_{out,m})}{N_{0}B_{m}}+1\big)$, $r_{m}^{(2)} = \log_2\big(-\frac{P_{max,m}\ln(1-p_{out,m})}{N_{0}B_{m}}+1\big)$, and $r_{m}^{(3)} = \frac{s_{m}}{B_{m}(T_{l}-\frac{c_{m}D_{m}}{f_{max,m}})}$, the following lemma can be proved.}

\begin{Lemma}\label{Lemma1}
Given any $\max\{r_{m}^{(1)},r_{m}^{(3)}\} \leq r_{m} \leq r_{m}^{(2)}$, the optimal transmission power {\color{black}of worker $m$} is
\begin{equation}
P_{m}^{*} = -\frac{N_{0}B_{m}(2^{r_{m}}-1)}{\ln(1-p_{out,m})},
\end{equation}
{\color{black}and} the optimal CPU frequency for local computation is given by
\begin{equation}
f^{*}_{m} = \max\bigg\{\frac{c_{m}D_{m}}{T_l-\frac{s_m}{r_{m}B_{m}}}, f_{min,m}\bigg\}.
\end{equation}
\end{Lemma}

\begin{proof}
Please see Appendix \ref{proofOfL1}.
\end{proof}
\begin{Remark}
Note that $[r_{m}^{(3)}$, $r_{m}^{(2)}]$ defines the feasible region of $r_{m}$. If $r_{m} < r_{m}^{(3)}$, the time duration requirement cannot be satisfied even with the maximum $f_{m}$. Similarly, the outage probability requirement cannot be satisfied even with the maximum $P_{m}$ if $r_{m} > r_{m}^{(2)}$. $r_{m}^{(1)}$ denotes the minimum transmission rate that worker $m$ will select. {\color{black}For any feasible transmission power, the outage probability requirement in (\ref{OP5}) is satisfied (i.e., $p_{out}(r_{m}) \leq p_{out,m}$) for any $r_{m} \leq r^{(1)}_{m}$. When $r_{m} < r^{(1)}_{m}$, worker $m$ can increase $r_{m}$ and decrease $f_{m}$ accordingly to accommodate the time duration requirement $T_{l}$.} Considering that the objective function of the optimization problem (\ref{OP5}) is a decreasing (increasing) function of $r_{m}$ ($f_{m}$), a smaller objective in (\ref{OP5}) can be achieved. As a result, we have the optimal transmission rate $r_{m}^{*} \geq r^{(1)}_{m}$.
\end{Remark}
With Lemma \ref{Lemma1} at hand, the optimization problem (\ref{OP5}) can be reformulated as follows.
\begin{equation}\label{OP4}
\begin{aligned}
&\underset{r_{m}}{\text{min}}~~\frac{\alpha_{m}c_{m}D_{m}}{2}z_{m}^{2}(r_{m}) - \frac{N_{0}s_{m}(2^{r_{m}}-1)}{\ln(1-p_{out,m})r_{m}} \\
& \text{s.t.}~~~~~~~\max\{r_{m}^{(1)},r_{m}^{(3)}\} \leq r_{m} \leq r_{m}^{(2)},
\end{aligned}
\end{equation}
in which $z_{m}(r_{m}) = \max\{\frac{c_{m}D_{m}}{T_l-\frac{s_m}{r_{m}B_{m}}}, f_{min,m}\}$. It can be verified that the objective in (\ref{OP4}) is convex and therefore, the subgradient methods \cite{boyd2003subgradient} can be adopted to solve the optimization problem (\ref{OP4}).

\section{Extension to the Heterogeneous Data Distribution Scenario}\label{SectionHeterogeneous}
\noindent The discussions in the previous sections consider the scenario with homogeneous data distribution across the workers. It has been shown that SignSGD fails to converge when the the data are heterogeneously distributed across the workers even when the workers can deliver their information without any error \cite{chen2019distributed}. The following example is provided for further illustration.

\begin{Example}\label{Example1}
Suppose that the $i$-th entry of worker $m$'s gradient is given as follows\footnote{Recall that in the homogeneous data distribution setting, the gradients of the workers are considered noisy versions of $\nabla F({\color{black}\boldsymbol{w}}^{({\color{black}k})})$. As a result, such a scenario as (\ref{HeterogeneousGradientExample}) happens with a small probability. {\color{black}In the heterogeneous data distribution setting, $p_{m,i}^{({\color{black}k})}$'s depend on the local datasets of the workers and may be very different from those in the homogeneous data distribution setting.}}
\begin{equation}\label{HeterogeneousGradientExample}
\nabla F_{m}({\color{black}\boldsymbol{w}}^{({\color{black}k})})_{i} =
\begin{cases}
-1, ~~~\text{if $1\leq m\leq M-1$}, \\
M, ~~~~~~~~\text{if $m=M$}.
\end{cases}
\end{equation}
In this case, we have
\begin{equation}
\begin{split}
sign(\nabla F({\color{black}\boldsymbol{w}}^{({\color{black}k})}))_{i} &= sign\bigg(\frac{1}{M}\sum_{m=1}^{M}\nabla F_{m}({\color{black}\boldsymbol{w}}^{({\color{black}k})})\bigg)_{i} \\
&= sign\bigg(\frac{1}{M}\bigg) = 1
\end{split}
\end{equation}
It can be easily verified that {\color{black}(c.f. (\ref{DefinitionofX_m}))}
\begin{equation}\label{eq30}
\begin{split}
P(X_{m,i}=1) &= P(sign({\color{black}\hat{\boldsymbol{g}}_{m}^{({\color{black}k})}})_{i}\neq sign(\nabla F({\color{black}\boldsymbol{w}}^{({\color{black}k})}))_{i}) \\
&=
\begin{cases}
1-p_{out}(r_{m}), ~~~~~~~\text{if $1\leq m\leq M-1$}, \\
p_{out}(r_{m}), ~~~~~~~~~~~~~~~\text{if $m=M$}.
\end{cases}
\end{split}
\end{equation}
Essentially, {\color{black}it can be observed that the aggregation result is wrong even free of channel errors.\footnote{\color{black}Note that (\ref{eq30}) is based on our worst case scenario analysis given by (\ref{worst-case-outage}). The same conclusion holds in general for other scenario (e.g., (\ref{discard-outage})) as well.}}
\end{Example}

In Example \ref{Example1}, it can be observed that for worker $m \in \{1,2,\cdots,M-1\}$, a smaller $p_{out}(r_{m})$ results in a larger $P(X_{m,i}=1)$ and smaller $(M - 2\mathbb{E}[Z_{i}])/M$, in sharp contrast with the homogeneous case. In this case, the parameter server obtains wrong aggregation results even if all the workers deliver their information without any error. In addition, when the data are heterogeneously distributed across the workers, the probability of such scenarios that lead to wrong aggregation is unknown since neither the parameter server nor the workers {\color{black}have} knowledge about the global objective function $F(\cdot)$. As a result, the convergence of Algorithm \ref{SIGNSGD} cannot be guaranteed. Therefore, it is of vital importance to develop an algorithm that can deal with heterogeneous data distribution across the workers.
With such consideration, a stochastic sign based algorithm (i.e., Algorithm \ref{noisySIGNSGD}), termed as Stochastic SignSGD with majority vote, is proposed. In particular, compared to Algorithm \ref{SIGNSGD}, there is a pre-processing step (i.e., step 3) in Algorithm \ref{noisySIGNSGD}, in which each worker projects each entry of the locally obtained gradient to -1 and +1 with certain probabilities, respectively. {\color{black}More specifically, each worker $m$ computes its gradient $\nabla F_{m}({\color{black}\boldsymbol{w}}^{({\color{black}k})})$ and obtains its expected transmission rate $r_{m}$ (either provided by the parameter server after solving (\ref{OP1}) or obtained by solving (\ref{OP5})), after which the outage probability $p_{out}(r_{m})$ can be estimated through (\ref{p_out}), and the pre-processing parameter $p_{m}^{i}$'s can be calculated accordingly.} The aforementioned issue is alleviated by the stochasticity of the projection. Taking $M=3$ in Example \ref{Example1} as an example, it can be verified that
\begin{equation}\label{p_mt}
\begin{split}
p_{m,i}^{({\color{black}k})} &= P\Big(sign({\color{black}\boldsymbol{g}}_{m}^{({\color{black}k})})_{i}=sign(\nabla F({\color{black}\boldsymbol{w}}^{({\color{black}k})}))_{i}\Big) \\&=
\begin{cases}
\frac{\frac{1}{2}-p_{out}(r_m)-b|\nabla F_{m}({\color{black}\boldsymbol{w}}^{({\color{black}k})})_{i}|}{1-2p_{out}(r_m)}, ~~~\text{if $1\leq m\leq 2$},\\
\frac{\frac{1}{2}-p_{out}(r_m)+b|\nabla F_{m}({\color{black}\boldsymbol{w}}^{({\color{black}k})})_{i}|}{1-2p_{out}(r_m)}, ~~~~~\text{if $m=3$}.
\end{cases}
\end{split}
\end{equation}
Plugging (\ref{p_mt}) into (\ref{PX_M}) yields
\begin{equation}
P(X_{m,i} = 1) =
\begin{cases}
\frac{1}{2}+b, ~~~\text{if $1\leq m\leq 2$}, \\
\frac{1}{2}-3b, ~~~~~\text{if $m=3$}.
\end{cases}
\end{equation}

Therefore, {\color{black}the probability of correct aggregation (c.f. (\ref{PcorrectAggregation})) is given by}
\begin{equation}
\begin{split}
&P\bigg(Z_{i} < \frac{3}{2}\bigg) = P\bigg(\sum_{m=1}^{3}X_{m,i} = 1\bigg) + P\bigg(\sum_{m=1}^{3}X_{m,i} = 0\bigg) \\
&= 2\bigg(\frac{1}{2}+b\bigg)\bigg(\frac{1}{2}-b\bigg)\bigg(\frac{1}{2}+3b\bigg)+\bigg(\frac{1}{2}-b\bigg)^2\bigg(\frac{1}{2}-3b\bigg)\\
&+\bigg(\frac{1}{2}-b\bigg)^2\bigg(\frac{1}{2}+3b\bigg)\\
&=\frac{1}{2} + \frac{1}{2}b-6b^3.
\end{split}
\end{equation}

It can be verified that when $0 < b < \frac{1}{\sqrt{12}}$, $P(Z_{i} < 3/2) > \frac{1}{2}$. {\color{black}That being said, the probability of correct aggregation is strictly larger than $\frac{1}{2}$ when $b$ is small enough,} based on which the convergence of Algorithm \ref{noisySIGNSGD} can be established \cite{jin2020stochastic}. {\color{black}For more general scenarios, the following lemma can be proved.}




\begin{algorithm}[!t]
\caption{Stochastic SignSGD with majority vote over wireless networks}
\label{noisySIGNSGD}
\begin{algorithmic}[1]
\State Input: initial weight: $w_0$; number of workers: $M$; learning rate: $\eta$.
\For{{\color{black}$k=0,1,\cdots,K$}}
    \State {\color{black} Each worker $m$ transmits $sign({\color{black}\boldsymbol{g}}^{({\color{black}k})}_{m})$ to the parameter server over wireless links, where ${\color{black}\boldsymbol{g}}^{({\color{black}k})}_{m}$ is obtained as follows: each worker $m$ first obtains its gradient $\nabla F_{m}({\color{black}\boldsymbol{w}}^{({\color{black}k})})$. Then, it estimates its outage probability $p_{out}(r_{m})$ and does the following pre-processing
\begin{equation}\label{pre-processing}
({\color{black}\boldsymbol{g}}^{({\color{black}k})}_{m})_{i} =
\begin{cases}
\hfill -sign(\nabla F_{m}({\color{black}\boldsymbol{w}}^{({\color{black}k})}))_{i}, &\text{with probability $p_{m}^{i}$},\\
\hfill sign(\nabla F_{m}({\color{black}\boldsymbol{w}}^{({\color{black}k})}))_{i}, &\text{with probability $1-p_{m}^{i}$},
\end{cases}
\end{equation}
where $p_{m}^{i}=\frac{\frac{1}{2}-p_{out}(r_m)-b|\nabla F_{m}({\color{black}\boldsymbol{w}}^{({\color{black}k})})_{i}|}{1-2p_{out}(r_m)}$ and $b$ is a parameter subject to design. Particularly, $0 < b < \frac{1-2p_{out}(r_m)}{2|\nabla F_{m}({\color{black}\boldsymbol{w}}^{({\color{black}k})})_{i}|}$ such that $p^{i}_{m} \in (0,\frac{1}{2})$.
}

    \State The parameter server obtains a noisy estimate (denoted by ${\color{black}\hat{\boldsymbol{g}}}^{({\color{black}k})}_{m}$) of the transmitted information $sign({\color{black}\boldsymbol{g}}^{({\color{black}k})}_{m})$ from each worker $m$ and sends the aggregated result ${\color{black}\tilde{\boldsymbol{g}}}^{({\color{black}k})} = sign\big(\frac{1}{M}\sum_{m=1}^{M}{\color{black}\hat{\boldsymbol{g}}}^{({\color{black}k})}_{m}\big)$ back to the workers.
    \State The workers update their local models
\begin{equation}
{\color{black}\boldsymbol{w}}^{({\color{black}k+1})} = {\color{black}\boldsymbol{w}}^{({\color{black}k})} - \eta{\color{black}\tilde{\boldsymbol{g}}}^{({\color{black}k})}.
\end{equation}
\EndFor
\end{algorithmic}
\end{algorithm}

{\color{black}
\begin{Lemma}\label{LemmaHeterogeneous2}
\color{black}
When $p_{out}(r_m) \leq \min_{i}\{\frac{1}{2}-b|\nabla F_{m}({\color{black}\boldsymbol{w}}^{({\color{black}k})})_{i}|\}$, in which $\nabla F_{m}({\color{black}\boldsymbol{w}}^{({\color{black}k})})_{i}$ is the $i$-th entry of the gradient $\nabla F_{m}({\color{black}\boldsymbol{w}}^{({\color{black}k})})$, we have
\begin{equation}\label{probone}
\begin{split}
&P\big({\color{black}\tilde{\boldsymbol{g}}}^{({\color{black}k})}_{i} = 1\big) \\&=
\begin{cases}
\frac{1}{2} + \frac{{M-1 \choose \frac{M-1}{2}}\sum_{m=1}^{M}b\nabla F_{m}({\color{black}\boldsymbol{w}}^{({\color{black}k})})_{i}}{2^{M}} + O\big(\frac{b^2}{2^{M}}\big),\\~~~~~~~~~~~~~~~~~~~~~~~~~~~~~~~~~~~~~~~~~~~~~\text{when $M$ is odd,} \\
\frac{1}{2} + \frac{\big[{M-1 \choose \frac{M-2}{2}}+{M-1 \choose \frac{M}{2}}\big]\sum_{m=1}^{M}b\nabla F_{m}({\color{black}\boldsymbol{w}}^{({\color{black}k})})_{i}}{2^{M+1}} + O\big(\frac{b^2}{2^{M}}\big), \\~~~~~~~~~~~~~~~~~~~~~~~~~~~~~~~~~~~~~~~~~~~~~\text{when $M$ is even.}
\end{cases}
\end{split}
\end{equation}
where ${\color{black}\tilde{\boldsymbol{g}}}^{({\color{black}k})}_{i}$ is the $i$-th entry of the aggregated result at the parameter server's side{\color{black}, and the parameter server breaks the tie when there is no winner in the majority vote by selecting $-1$ and $+1$ uniformly at random.}
\end{Lemma}
\begin{proof}
Please see Appendix \ref{ProofofLemma32}.
\end{proof}

\begin{Remark}\label{Remark6}
(\ref{probone}) measures the probability of the $i$-th entry of the aggregated result being 1. If the second term in (\ref{probone}) dominates (i.e., $b$ is sufficiently small), $P\big({\color{black}\tilde{\boldsymbol{g}}}^{({\color{black}k})}_{i} = 1\big) > \frac{1}{2}$ when $\sum_{m=1}^{M}\nabla F_{m}({\color{black}\boldsymbol{w}}^{({\color{black}k})})_{i} > 0$; $P\big({\color{black}\tilde{\boldsymbol{g}}}^{({\color{black}k})}_{i} = 1\big) < \frac{1}{2}$ when $\sum_{m=1}^{M}\nabla F_{m}({\color{black}\boldsymbol{w}}^{({\color{black}k})})_{i} < 0$. That being said, the probability of wrong aggregation is always smaller than 1/2. {\color{black}As a result, the convergence of Algorithm \ref{noisySIGNSGD} can be established.}

{\color{black}When the workers compute their local gradients (denoted by $\nabla \tilde{F}_{m}(\boldsymbol{w}^{(k)})$'s) over a mini-batch of their local training data, {\color{black}$\nabla F_{m}(\boldsymbol{w}^{(k)})_{i}$ in Lemma 3 will be replaced by $\nabla \tilde{F}_{m}(\boldsymbol{w}^{(k)})_{i}$. That being said, it can be shown that $P(sign(\tilde{\boldsymbol{g}}_{i}^{(k)})) \neq sign(\sum_{m=1}^{M}\nabla \tilde{F}_{m}(\boldsymbol{w}^{(k)}))_{i} <\frac{1}{2}$, when $b$ is sufficiently small.} In this case, $\sum_{m=1}^{M}\nabla \tilde{F}_{m}(\boldsymbol{w}^{(k)})$ can be understood as a noisy version of $\sum_{m=1}^{M}\nabla F_{m}(\boldsymbol{w}^{(k)})$ due to the training data sampling. Following the same assumption in \cite{bernstein2018signsgd1} that the sampling noise is symmetric with zero mean, $P(sign(\sum_{m=1}^{M}\nabla \tilde{F}_{m}(\boldsymbol{w}^{(k)}))_{i}\neq sign(\sum_{m=1}^{M}\nabla F_{m}(\boldsymbol{w}^{(k)}))_{i}) < \frac{1}{2}$ always holds. As a result, it can also be shown that the probability of wrong aggregation is always smaller than 1/2.}
\end{Remark}

Given Lemma \ref{LemmaHeterogeneous2}, the following theorem can be proved according to \cite{jin2020stochastic}.
\begin{Theorem}\label{Theorem2}
Suppose Assumption \ref{A2} is satisfied and the learning rate is set as $\eta=\frac{1}{\sqrt{dT}}$, then by running Algorithm \ref{noisySIGNSGD} for ${\color{black}K}$ iterations, we have
\begin{equation}
\begin{split}
\frac{1}{T}\sum_{t=1}^{T}c||\nabla F({\color{black}\boldsymbol{w}}^{({\color{black}k})})||_{1} &\leq \frac{\mathbb{E}[F({\color{black}\boldsymbol{w}}^{(0)}) - F({\color{black}\boldsymbol{w}}^{({\color{black}k+1})})]\sqrt{d}}{{\color{black}\sqrt{K}}},
\end{split}
\end{equation}
where $c$ is some positive constant.
\end{Theorem}

}

\begin{figure*}[ht]
\begin{subfigure}{.33\textwidth}
  \centering
  \includegraphics[width=.95\linewidth]{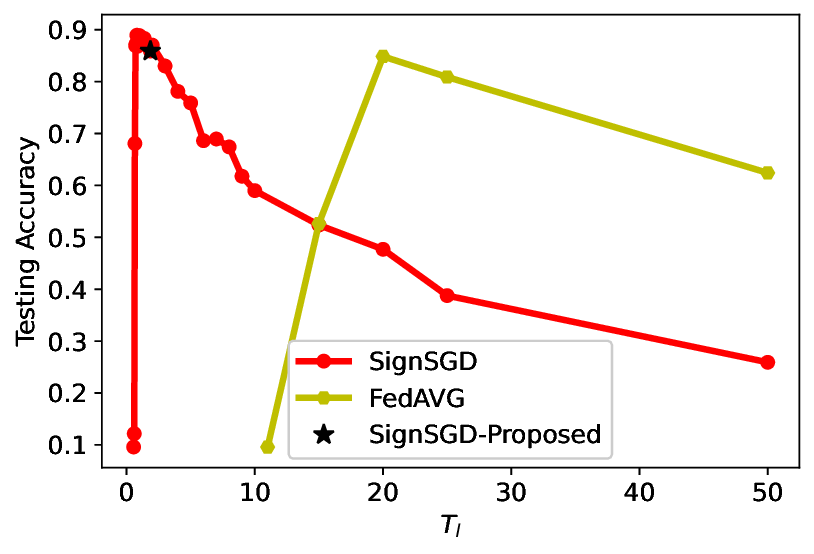}
  \caption{$P_{m}=0.005$W}
  \label{fig:sub-first}
\end{subfigure}
\begin{subfigure}{.33\textwidth}
  \centering
  \includegraphics[width=.95\linewidth]{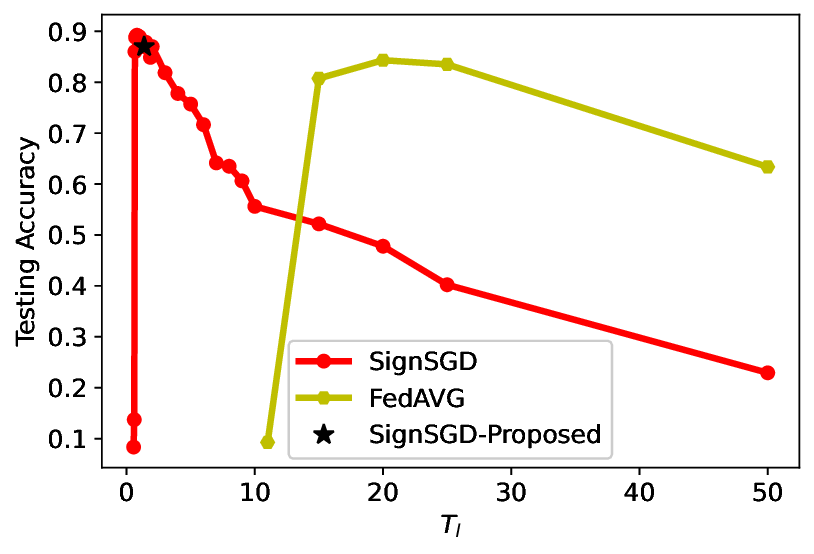}
  \caption{$P_{m}=0.01$W}
  \label{fig:sub-second}
\end{subfigure}
\begin{subfigure}{.33\textwidth}
  \centering
  \includegraphics[width=.95\linewidth]{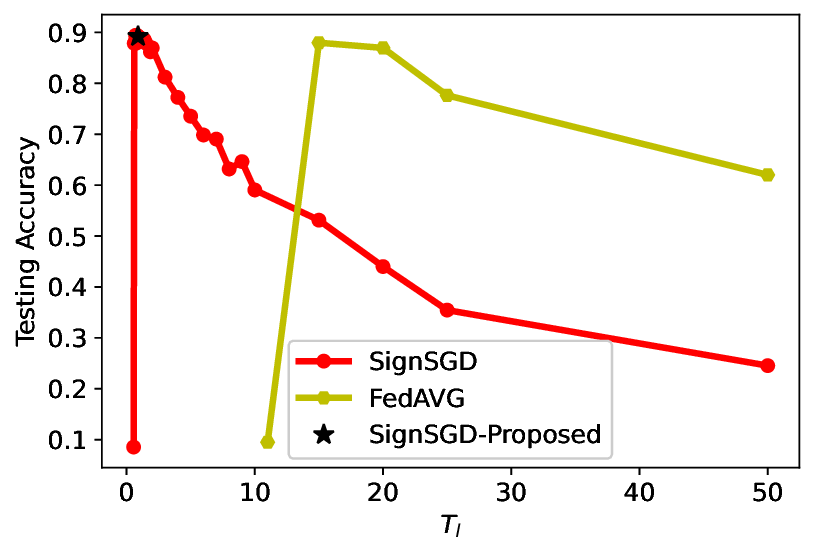}
  \caption{$P_{m}=0.05$W}
  \label{fig:sub-third}
\end{subfigure}
\caption{The impact of communication round duration $T_{l}$ on the learning algorithms given different transmission power $P_{m}$.}
\label{fig:fig}
\end{figure*}



{\color{black}It is worth mentioning that according to Lemma \ref{LemmaHeterogeneous2} and Theorem \ref{Theorem2}, the convergence of Algorithm \ref{noisySIGNSGD} is based on the condition $p_{out}(r_m) \leq \min_{i}\{\frac{1}{2}-b|\nabla F_{m}({\color{black}\boldsymbol{w}}^{({\color{black}k})})_{i}|\}$. Therefore, in order to optimize the learning performance, the corresponding constraint needs to be added to the optimization problem. In the meantime, $b$ should also be optimized. However, since $\nabla F_{m}({\color{black}\boldsymbol{w}}^{({\color{black}k})})$'s are unknown to the server, optimizing $b$ and the learning performance of Algorithm \ref{noisySIGNSGD} is highly non-trivial and left as our future work.}


{\color{black}With such consideration, in this work, we mainly consider the energy consumption minimization problem given predetermined $b$,\footnote{In the implementation of Algorithm \ref{noisySIGNSGD}, it is possible that the predetermined $b\geq \frac{1-2p_{out}(r_m)}{2|\nabla F_{m}({\color{black}\boldsymbol{w}}^{({\color{black}k})})_{i}|}$ for the $i$-th entry of worker $m$'s gradient such that $p_{m}^{i} \leq 0$. In this case, we round $p_{m}^{i}$ to 0, and it can be verified that (\ref{pre-processing}) is reduced to $({\color{black}\boldsymbol{g}}^{({\color{black}k})}_{m})_{i} = sign(\nabla F_{m}({\color{black}\boldsymbol{w}}^{({\color{black}k})}))_{i}$. That being said, Algorithm \ref{SIGNSGD} is a special case of Algorithm \ref{noisySIGNSGD} where $b$ is large enough.} time duration requirement $T_{l}$ and the outage probability requirement $p_{out,m}$ for each worker. It can be observed from (\ref{probone}) that the probability of correct aggregation is independent of the outage probability $p_{out}(r_m)$. In the meantime, according to (\ref{OP5}), the feasible region of the energy consumption minimization problem with a smaller $p_{out,m}$ is a subset of that with a larger $p_{out,m}$. As a result, it is optimal to select $p_{out,m} = \min_{i}\{\frac{1}{2}-b|\nabla F_{m}({\color{black}\boldsymbol{w}}^{({\color{black}k})})_{i}|\}$ at the ${\color{black}k}$-th communication round. In this case, worker $m$ has to obtain $\nabla F_{m}({\color{black}\boldsymbol{w}}^{({\color{black}k})})$ before computing $p_{out,m}$. That being said, it has to finish the local computation of the gradients before solving the energy consumption minimization problem. {\color{black}To handle this challenge,} we consider a pre-determined local computation CPU frequency $f_m$ for each worker $m$'s energy consumption minimization problem (which can be realized by setting $f_{min,m}=f_{max,m}=f_m$ in (\ref{OP5})).}

To this end, during each communication round, each worker $m$ first computes its gradient $\nabla F_{m}({\color{black}\boldsymbol{w}}^{({\color{black}k})})$ and determines the optimal outage probability requirement $p_{out,m}$ as above. By solving the energy consumption minimization problem (\ref{OP5}), it obtains the communication parameters $P_{m}$ and $r_{m}$. Then, each worker $m$ estimates its outage probability $p_{out}(r_{m})$ (e.g., through (\ref{p_out})) and performs the pre-processing step, after which the processed information is transmitted to the parameter server over wireless links.

\section{Simulation Results}\label{Simulations}
\color{black}
\noindent In this section, we examine the performance of the proposed methods through extensive simulations. We implement the proposed method with a two-layer fully connected neural network {\color{black}with 101,770 trainable parameters} on the well-known MNIST dataset that consists of 10 categories ranging from digit ``0" to ``9" and a total of 60,000 training samples and 10,000 testing samples. In this case, the size of updates for each worker is $s_{m}=101,770$ bits for each communication round. It is assumed that there are 31 workers that collaboratively train a global model. For all the workers, {\color{black}similar to \cite{tran2019federated} and \cite{chen2020joint},} we set $\alpha_{m}=2\times10^{-28}$; $c_m=20$ cycles/bit; $D_m=5\times10^{7}$ bits; $N_{0}=10^{-8}$ W/Hz; $B_m=180$ kHz. {\color{black} In the training process, all the workers use a mini-batch of size 16 to compute their local gradients.\footnote{\color{black}Note that similar results are obtained for the full batch scenario where each worker evaluates its local gradient over the whole local dataset, which are omitted in the interest of space.}} In the scenario with homogeneous data distribution across the workers, each worker randomly samples 2000 training samples from the training dataset. In the scenario with heterogeneous data distribution across the workers, {\color{black}two heterogeneous data distribution schemes are considered. In the first scheme, we consider the extreme case where the whole training dataset is divided into 31 subsets (each assigned to one worker), and each subset contains training data with one label only. In the second scheme, we adopt the Dirichlet distribution $\text{Dir}(\alpha)$ scheme \cite{lin2020ensemble} to synthesize label distribution skew in experiments, where $\text{Dir}(\alpha)$ is the symmetric Dirichlet distribution with the concentration parameter $\alpha$ and a smaller $\alpha$ indicates more severe heterogeneity. The data allocation procedure consists of the following two steps: (1) for each worker $m$, a random vector $\boldsymbol{q}_m \sim \text{Dir}(\alpha)$ is drawn where $\boldsymbol{q}_m = [q_{m,1}, \cdots, q_{m,N}]^{\top}$ and $q_{m,n}$ indicates the {\color{black}ratio} of worker $m$'s local training dataset from the $n$-th class; (2) the training data are assigned to each worker $m$, with the number of training samples from the $n$-th class proportional to $q_{m,n}$.}

We compare the proposed method with a commonly adopted baseline for federated learning algorithms: FedAVG \cite{mcmahan2017communication} in which the model updates are transmitted in full precision. In FedAVG, the workers run $\tau$ local SGD steps before sharing their model updates with the parameter server. {\color{black}In our simulations, we examine the performance of FedAVG for $\tau\in \{1,5,10,20\}$ and present the results with the best testing accuracy.} Note that our aforementioned analyses on SignSGD consider the worst case scenario, i.e., all entries of the transmitted updates are incorrectly decoded in an outage. However, the worst case scenario for FedAVG is not easy to define, and it may fail to converge if the workers in outage share arbitrary information with the parameter server \cite{chen2017distributed}. With such consideration, we assume in this section that the parameter server can detect the outage and discard the packages in outage.
\begin{figure*}[ht]
\begin{minipage}[t]{0.33\linewidth}
\centering
\includegraphics[width=1\textwidth]{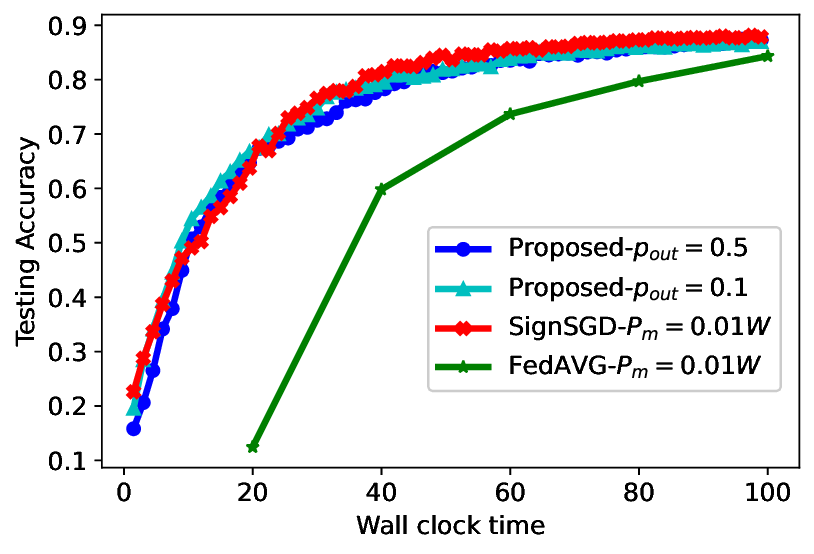}
\caption{Testing Accuracy of the Learning Algorithms.}
\label{TestingIID_P1_energyminimization}
\end{minipage}
\begin{minipage}[t]{0.33\linewidth}
\centering
\includegraphics[width=1\textwidth]{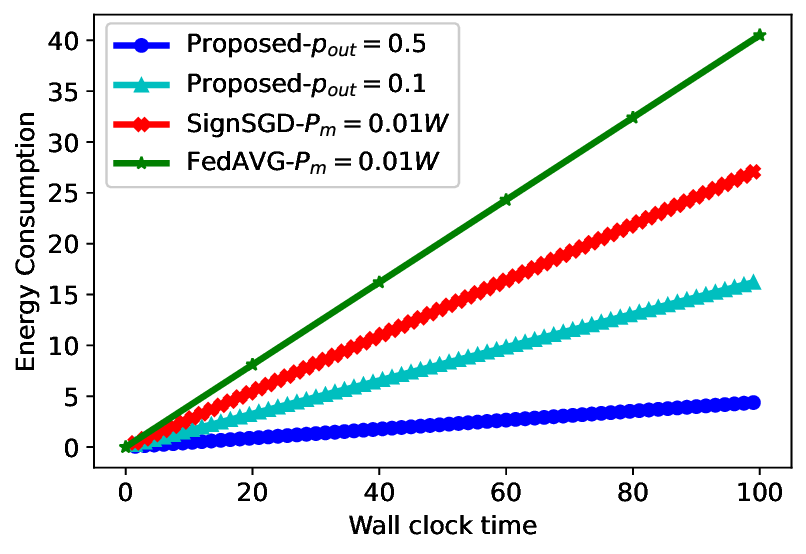}
\caption{Energy Consumption of the Learning Algorithms.}
\label{EnergyIID_P1_energyminimization}
\end{minipage}
\begin{minipage}[t]{0.33\linewidth}
\centering
\includegraphics[width=1\textwidth]{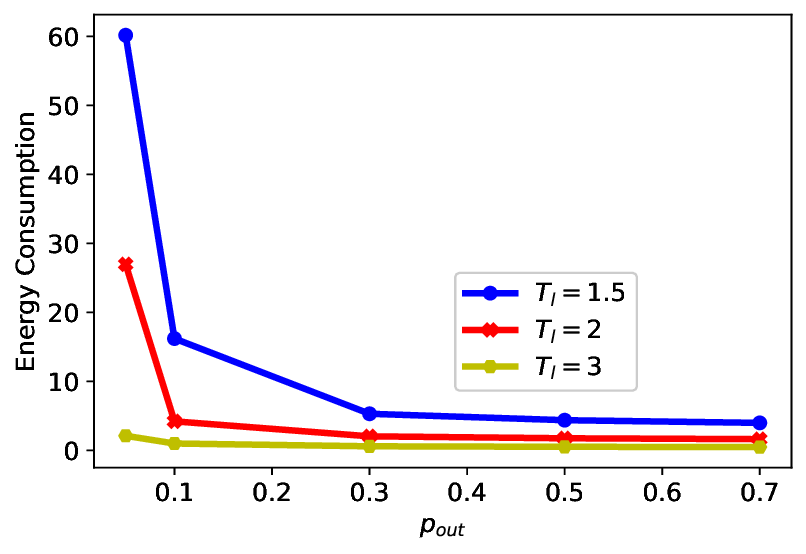}
\caption{Energy Consumption of the Proposed Method for Different $p_{out}$ and $T_{l}$.}
\label{EnergyIID_P1_Tlpout}
\end{minipage}
\end{figure*}
\subsection{Learning Performance Optimization Given Energy Consumption Constraint: Homogeneous}\label{Section:ExperimentHomo}
\noindent In this subsection, we examine the impact of the transmission power $P_{m}$ and the time duration for each communication round $T_{l}$. The CPU frequencies of all the workers are set to $2$ GHz. For SignSGD, the energy consumption upper limit is set as $E_m = 100$J for each communication round. For FedAVG, the workers are assumed to have unlimited battery (i.e., $E_{m} = \infty$) such that the outage probabilities of the workers are not limited by their energy consumption constraints. {\color{black}Fig. \ref{fig:fig}} show the testing accuracy of Algorithm \ref{SIGNSGD} with different $P_{m}$ and $T_{l}$, given a total training time $T_{total}=100$s. For SignSGD, the transmission rates $r_{m}$'s are given by (\ref{OptimalTransmissionRate}), while the configurations of ``SignSGD-Proposed" are given by the solution of (\ref{OP1}). For FedAVG, the transmission rates are selected such that the outage probabilities are minimized given the time duration for each communication round.

It can be seen from {\color{black}Fig. \ref{fig:fig}} that the proposed method for SignSGD works close to the optimal operation point for all the examined scenarios, which validates its effectiveness. Besides, it is shown that as $T_{l}$ increases, the learning performance of both algorithms first increases and then decreases. According to (\ref{OptimalTransmissionRate}), when $T_{l}$ increases, $r_{m}$ decreases and therefore the outage probability $p_{out}(r_{m})$ also decreases. However, in the meantime, as $T_{l}$ increases, the number of communication rounds decreases given the fixed training time. As a result, when the outage probability has a larger impact on the learning performance, increasing $T_{l}$ results in better performance. When $T_{l}$ is larger than a certain critical value, the number of communication rounds plays a more important role and therefore increasing $T_{l}$ leads to worse performance.

\begin{table}[t]
\caption{Testing Accuracy of the Learning Algorithms}
\vspace{-0.1in}
\label{AccuracyHomo_table}
\begin{center}
\begin{tabular}{ | m{3cm} | m{0.8cm} | m{0.8cm}|m{0.8cm}|}
\hline
$P_{m}$ & 0.005W & 0.01W & 0.05W \\
\hline
SignSGD-Proposed & 85.89\% & 87.01\% &89.19\%\\
\hline
FedAVG (optimal $T_{l}$) & 84.85\% & 84.31\% &87.96\%\\
\hline
\end{tabular}
\end{center}
\vspace{-0.1in}
\end{table}

\begin{table}[t]
\caption{Energy Consumption of the Learning Algorithms}
\vspace{-0.1in}
\label{EnergyHomo_table}
\begin{center}
\begin{tabular}{ | m{3cm} | m{0.8cm} | m{0.8cm}|m{0.8cm}|}
\hline
$P_{m}$ & 0.005W & 0.01W & 0.05W \\
\hline
SignSGD-Proposed & 21.56J & 29.11J&44.62J\\
\hline
FedAVG (optimal $T_{l}$) & 40.25J & 40.5J & 49.5J\\
\hline
\end{tabular}
\end{center}
\vspace{-0.1in}
\end{table}

Table \ref{AccuracyHomo_table} and Table \ref{EnergyHomo_table} show the testing accuracy and energy consumption of SignSGD with the proposed operation point and FedAVG with the optimal $T_{l}$. It can be observed that SignSGD outperforms FedAVG in both testing accuracy and energy consumption. Moreover, we note that there is no existing method in literature that can obtain the optimal $T_{l}$ and the number of local iterations $\tau$ for FedAVG.

\subsection{Energy Minimization Given Learning Performance Constraint: Homogeneous}
\noindent In this subsection, the performance of SignSGD with energy minimization is examined. We set $f_{min,m}=0.2$ GHz; $f_{max,m}=2$ GHz; $P_{min,m}=0$ W; $P_{max,m}=0.01$ W. For FedAVG, the CPU frequency $f_{m}$ and the transmission power $P_{m}$ are set to 2 GHz and 0.01 W, respectively, for all the workers such that the minimum outage probability can be achieved. For SignSGD and the proposed method, we set $T_{l}=1.5$s; while for FedAVG, we examine the performance of different $T_{l} \in \{3, 4, 5, 6, 7, 8, 9, 10, 15, 20, 25, 50\}$ and present the results with the best testing accuracy. For the proposed method, we set the same outage probability constraints for all the workers, i.e., $p_{out,m} = p_{out}, \forall m$.

Fig. \ref{TestingIID_P1_energyminimization}-\ref{EnergyIID_P1_energyminimization} show the testing accuracy and energy consumption of the learning algorithms. It can be observed that the proposed method with $p_{out}\in\{0.1,0.5\}$ outperforms FedAVG in both testing accuracy and energy consumption. Moreover, it is shown that increasing the outage probability of SignSGD does not change the testing accuracy much. This is because most of the workers share the correct signs in the homogeneous data distribution scenario, and therefore, the parameter server can update the model in the correct direction even when some workers are in the outage. Fig. \ref{EnergyIID_P1_Tlpout} shows the energy consumption of the proposed method for different $p_{out}$ and $T_{l}$. It can be observed that as $p_{out}$ and $T_{l}$ increase, the energy consumption decreases. This is because the feasible region of (\ref{OP5}) corresponding to a smaller $p_{out}$ and $T_{l}$ is a subset of that of (\ref{OP5}) corresponding to a larger $p_{out}$ and $T_{l}$.

\begin{table}[t]
\caption{Testing Accuracy of the Learning Algorithms}
\vspace{-0.1in}
\label{AccuracyHetero_table}
\begin{center}
\begin{tabular}{ | m{4.5cm} | m{0.8cm} | m{0.8cm}|m{0.8cm}|}
\hline
CPU Frequency $f$ & 1GHz & 2GHz & 3GHz \\
\hline
Algorithm \ref{noisySIGNSGD}-$b=100$, $T_{total}=250$s & 86.86\% & 88.32\% &87.23\%\\
\hline
SignSGD, $T_{total}=300$s & 64.07\% & 60.5\% &60.70\%\\
\hline
FedAVG (optimal $T_{l}$), $T_{total}=300$s  & 82.10\% & 82.95\% &86.19\%\\
\hline
\color{black}SignSGD, $p_{out}=0.1$, $T_{total}=300$s &  \multicolumn{3}{|c|}{61.23\%}\\
\hline
\color{black}Algorithm \ref{noisySIGNSGD}-$b=100$, $T_{total}=250$s, optimal $f$, $p_{out} = 0.1$ & \multicolumn{3}{|c|}{86.89\%} \\
\hline
\end{tabular}
\end{center}
\vspace{-0.1in}
\end{table}

\begin{table}[t]
\caption{Energy Consumption of the Learning Algorithms}
\vspace{-0.1in}
\label{EnergyHetero_table}
\begin{center}
\begin{tabular}{ | m{4.5cm} | m{0.8cm} | m{0.8cm}|m{0.8cm}|}
\hline
CPU Frequency & 1GHz & 2GHz & 3GHz \\
\hline
Algorithm \ref{noisySIGNSGD}-$b=100$, $T_{total}=250$s & 20.65J & 74.47J&158.78J\\
\hline
SignSGD, $T_{total}=300$s  & 25.0J & 90.0J&191.67J\\
\hline
FedAVG (optimal $T_{l}$), $T_{total}=300$s  & 23.1J & 93.45J&235.83J\\
\hline
\color{black}SignSGD, $p_{out}=0.1$, $T_{total}=300$s  & \multicolumn{3}{|c|}{16.45J} \\
\hline
\color{black}Algorithm \ref{noisySIGNSGD}-$b=100$, $T_{total}=250$s, optimal $f$, $p_{out} = 0.1$ & \multicolumn{3}{|c|}{13.65J} \\
\hline
\end{tabular}
\end{center}
\vspace{-0.1in}
\end{table}

\subsection{Energy Minimization Given Learning Performance Constraint: Heterogeneous}\label{heterogeneousoptimization}
\noindent In this subsection, the performance of Algorithm \ref{noisySIGNSGD} is examined. {\color{black}We first consider the extreme case where each worker stores only training data from one label}. We set $P_{min,m}=0$ W and $P_{max,m}=0.05$ W for all the workers. Table \ref{AccuracyHetero_table} and Table \ref{EnergyHetero_table} show the testing accuracy and energy consumption of the learning algorithms, respectively. The time duration for each communication round is set to $T_{l} = 1.5$s for both Algorithm \ref{noisySIGNSGD} and SignSGD. For FedAVG, we examine the performance of different $T_{l}$'s and present the results with the best testing accuracy. {\color{black}For SignSGD and FedAVG, we set the total training time $T_{total}=300$s, while for Algorithm \ref{noisySIGNSGD}, we present the results with a total training time of $T_{total} = 250$s. For SignSGD and FedAVG, we first set the transmission power $P_{m}=0.05W$ for all the workers (such that the minimum outage probability can be achieved) and examine their performance given different CPU frequencies (i.e., 1 GHz, 2 GHz and 3 GHz). In addition to the pre-determined CPU frequency cases where the workers obtain their gradients before computing their outage probabilities, we also examined the scenario where each worker optimizes the CPU frequency given a pre-determined outage probability requirement $p_{out,m}=0.1$ for Algorithm \ref{noisySIGNSGD} with $f_{min,m}=0.2$ GHz and $f_{max,m}=3$ GHz. Correspondingly, the performance of SignSGD is examined in the same setting. For FedAVG with optimal $T_{l}$, we note that the outage probabilities of the workers are larger than 0.1 for all the examined CPU frequencies. In this sense, adding the outage probability requirement essentially leads to the same results as $P_{m} = P_{max,m} = 0.05$ W and $f_{m} = f_{max,m} = 3$ GHz, given that the workers try their best to minimize their outage probabilities. It can be observed from Table \ref{AccuracyHetero_table} and Table \ref{EnergyHetero_table} that Algorithm \ref{noisySIGNSGD} outperforms FedAVG and SignSGD in both testing accuracy and energy consumption given a shorter training time. That being said, Algorithm \ref{noisySIGNSGD} improves the learning performance and reduces the energy consumption {\color{black}simultaneously, which} demonstrates its effectiveness. {\color{black}Moreover, we note that the condition $p_{out}(r_m) \leq \min_{i}\{\frac{1}{2}-b|\nabla F_{m}({\color{black}\boldsymbol{w}}^{({\color{black}k})})_{i}|\}$ may be violated for a pre-determined $p_{out,m}=0.1$. In this case, the pre-processing probabilities $p_{m}^{i}$ and $1-p_{m}^{i}$ in (\ref{pre-processing}) may fall out of the range $[0,1]$. We round them to 1 if they are positive and 0 otherwise.} Table \ref{AccuracyHetero_table} and Table \ref{EnergyHetero_table} show that the energy consumption of the workers can be further reduced by optimizing the CPU frequency $f_{m}$'s with a negligible loss in testing accuracy.}

{\color{black}We also examine the impact of data heterogeneity {\color{black}(i.e., $\alpha$)} in the Dirichlet distribution scheme. Table \ref{AccuracyHeteroDirich_table} and Table \ref{EnergyHeteroDirich_table} compare the testing accuracy and the energy consumption of the learning algorithms, respectively, given {\color{black}$f_{m}=2$} GHz for all the workers. {\color{black}It can be observed that the performance of SignSGD and FedAVG improves as $\alpha$
increases, which corresponds to more homogeneous data distribution. Meanwhile, Algorithm \ref{noisySIGNSGD} outperforms SignSGD and FedAVG in both testing accuracy and energy consumption in all the examined scenarios, which demonstrates its effectiveness in the presence of different levels of data heterogeneity. Moreover, it is worth mentioning that the optimal $T_{l}$ and number of local iterations for FedAvg are different given different $\alpha$, which leads to different energy consumption.}}

\begin{table}[ht]
\color{black}
\caption{Testing Accuracy of the Learning Algorithms}
\vspace{-0.1in}
\label{AccuracyHeteroDirich_table}
\begin{center}
\begin{tabular}{ | m{3cm} | m{0.8cm} | m{0.8cm}|m{0.8cm}|m{0.8cm}|}
\hline
$\alpha$ & 0.01 & 0.1 & 1 & 10\\
\hline
Algorithm \ref{noisySIGNSGD}-$b=100$, $T_{total}=250$s & 93.02\% & 92.35\% &92.31\%&92.63\%\\
\hline
SignSGD, $T_{total}=300$s & 82.99\% & 86.83\% &87.64\%&88.11\%\\
\hline
FedAVG (optimal $T_{l}$), $T_{total}=300$s  & 87.50\% & 89.38\% &91.81\%&91.88\%\\
\hline
\end{tabular}
\end{center}
\vspace{-0.1in}
\end{table}

\begin{table}[ht]
\color{black}
\caption{Energy Consumption of the Learning Algorithms}
\vspace{-0.1in}
\label{EnergyHeteroDirich_table}
\begin{center}
\begin{tabular}{ | m{3cm} | m{0.8cm} | m{0.8cm}|m{0.8cm}|m{0.8cm}|}
\hline
$\alpha$ & 0.01 & 0.1 & 1 & 10\\
\hline
Algorithm \ref{noisySIGNSGD}-$b=100$, $T_{total}=250$s & 74.70J & 74.70J&74.68J&74.66J\\
\hline
SignSGD, $T_{total}=300$s  & 90.00J & 90.00J&90.00J&90.00J\\
\hline
FedAVG (optimal $T_{l}$), $T_{total}=300$s  & 142.50J & 138.75J&165.00J&150.00J\\
\hline
\end{tabular}
\end{center}
\vspace{-0.1in}
\end{table}

\color{black}
\subsection{The Impact of Outage Probability Estimation Errors}
\noindent In the previous subsections, we assume that the channel gain follows Rayleigh distribution and the outage probability can be captured by (\ref{p_out}). In practice, however, the fading over the wireless channels is usually more complicated than Rayleigh fading. In this subsection, we demonstrate the robustness of the proposed method against outage probability estimation errors by examining the scenario where the actual outage probability is different from (\ref{p_out}).

\begin{table}[t]
\caption{Testing Accuracy of SignSGD: Homogeneous}
\vspace{-0.1in}
\label{homounreliabletableaccuracy}
\begin{center}
\begin{tabular}{ | m{2cm} | m{0.8cm} | m{0.8cm}|m{0.8cm}|}
\hline
$P_{m}$ & 0.005W & 0.01W & 0.05W \\
\hline
$\Delta = 0$ & 85.89\% & 87.01\% &89.19\%\\
\hline
$\Delta = 10\%$& 85.65\% & 88.26\% &89.29\%\\
\hline
$\Delta = 30\%$& 86.08\% & 88.47\% &89.17\%\\
\hline
$\Delta = 50\%$& 86.42\% & 88.56\% &89.25\%\\
\hline
\end{tabular}
\end{center}
\vspace{-0.1in}
\end{table}

For the homogeneous data distribution scenario, given the theoretically derived outage probability $p_{out}(r_{m})$, we assume that the actual outage probability is uniformly distributed in the interval $[p_{out}(r_{m}), p_{out}(r_{m})\times (1+\Delta)]$. That being said, the channel conditions are worse than the ones that each worker uses for parameter configuration. We note that this is conservative since the performance of the learning algorithms is supposed to improve as the probability of outage decreases. Table \ref{homounreliabletableaccuracy} shows the testing accuracy of SignSGD with the configurations given by the solution of (\ref{OP1}). The other parameters are the same as those in Section \ref{Section:ExperimentHomo}. It can be observed that the SignSGD with $\Delta\in\{10\%, 30\%, 50\%\}$ achieves comparable performance in testing accuracy to the $\Delta=0$ case. We note that, according to (\ref{OP5}), the energy consumption is independent of $\Delta$. Therefore, the energy consumption remains the same for different $\Delta$.

For the heterogeneous data distribution scenario, it is shown in Section \ref{heterogeneousoptimization} that decreasing the outage probability is not necessarily better for SignSGD. With such consideration, we also examine the scenario that the actual outage probability is uniformly distributed in the interval $[p_{out}(r_{m})\times(1-\Delta), p_{out}(r_{m})\times (1+\Delta)]$. Table \ref{heterounreliabletableaccuracy} shows the testing accuracy of Algorithm \ref{noisySIGNSGD} with different $\Delta$, given a total training time of $T_{total} = 300$s. The other parameters are the same as those in Section \ref{heterogeneousoptimization}. It can be observed that Algorithm \ref{noisySIGNSGD} achieves similar performance in testing accuracy for different $\Delta$, which demonstrates its robustness against outage probability estimation errors.

\begin{table}[t]
\caption{Testing Accuracy of Algorithm \ref{noisySIGNSGD}: Heterogeneous}
\vspace{-0.1in}
\label{heterounreliabletableaccuracy}
\begin{center}
\begin{tabular}{ | m{2cm} | m{0.8cm} | m{0.8cm}|m{0.8cm}|}
\hline
CPU Frequency & 1GHz & 2GHz & 3GHz \\
\hline
$\Delta = 0$ & 88.11\% & 88.85\% &88.61\%\\
\hline
$\Delta = 10\%$& 88.39\% & 87.23\% &88.55\%\\
\hline
$\Delta = 30\%$& 88.34\% & 87.02\% &88.60\%\\
\hline
$\Delta = 50\%$& 88.55\% & 87.35\% &88.74\%\\
\hline
$\Delta = \pm10\%$& 87.87\% &87.64\% &88.74\%\\
\hline
$\Delta = \pm30\%$& 87.90\% &88.11\% &88.62\%\\
\hline
$\Delta = \pm50\%$& 87.79\% & 88.18\% &87.60\%\\
\hline
\end{tabular}
\end{center}
\vspace{-0.1in}
\end{table}


\color{black}
\section{Conclusions}\label{Conclusion}
\noindent In this work, the implementation of FL algorithms over wireless networks is studied. In particular, considering that the workers have limited batteries, two optimization problems concerning the learning performance and the energy consumption of the workers are formulated and solved for appropriate local processing and communication parameter configuration. Furthermore, since SignSGD fails to converge in the scenario with heterogeneous data distribution across the workers, a stochastic sign based algorithm that can deal with data heterogeneity across the workers is proposed and the corresponding energy minimization problem is solved. It is shown that the proposed algorithm improves the learning performance with less energy consumption for the workers. The simulation results demonstrate the effectiveness of the proposed method.

\color{black}
\appendices
\section{Proof of Theorem \ref{T1}}\label{ProofOfT1}
\begin{proof}
According to Assumption \ref{A2}, we have
\begin{equation}\label{convergencee1}
\begin{split}
&F({\color{black}\boldsymbol{w}}^{({\color{black}k+1})}) - F({\color{black}\boldsymbol{w}}^{({\color{black}k})}) \\
&\leq <\nabla F({\color{black}\boldsymbol{w}}^{({\color{black}k})}), {\color{black}\boldsymbol{w}}^{({\color{black}k+1})}-{\color{black}\boldsymbol{w}}^{({\color{black}k})}> + \frac{L}{2}||{\color{black}\boldsymbol{w}}^{({\color{black}k+1})}-{\color{black}\boldsymbol{w}}^{({\color{black}k})}||^2 \\
& =-\eta <\nabla F({\color{black}\boldsymbol{w}}^{({\color{black}k})}), sign(\sum_{m=1}^{M}{\color{black}\hat{\boldsymbol{g}}}^{({\color{black}k})}_{m})> + \frac{L}{2}||\eta sign(\sum_{m=1}^{M}{\color{black}\hat{\boldsymbol{g}}}^{({\color{black}k})}_{m})||^2 \\
& = -\eta <\nabla F({\color{black}\boldsymbol{w}}^{({\color{black}k})}), sign(\sum_{m=1}^{M}{\color{black}\hat{\boldsymbol{g}}}^{({\color{black}k})}_{m})> + \frac{L\eta^2d}{2} \\
& = \eta ||\nabla F({\color{black}\boldsymbol{w}}^{({\color{black}k})})||_{1} + \frac{L\eta^2d}{2} - 2\eta\sum_{i=1}^{d}|\nabla F({\color{black}\boldsymbol{w}}^{({\color{black}k})})_{i}|\times\\
&\mathds{1}_{sign(\sum_{m=1}^{M}{\color{black}\hat{\boldsymbol{g}}}^{({\color{black}k})}_{m})_{i} = sign(\nabla F({\color{black}\boldsymbol{w}}^{({\color{black}k})})_{i})},
\end{split}
\end{equation}
in which $\nabla F({\color{black}\boldsymbol{w}}^{({\color{black}k})})_{i}$ is the $i$-th entry of the vector $\nabla F({\color{black}\boldsymbol{w}}^{({\color{black}k})})$. Taking expectation on both sides yields
\begin{equation}
\begin{split}
&\mathbb{E}[F({\color{black}\boldsymbol{w}}^{({\color{black}k})})-F({\color{black}\boldsymbol{w}}^{({\color{black}k+1})})] \geq  -\eta||\nabla F({\color{black}\boldsymbol{w}}^{({\color{black}k})})||_{1} - \frac{L\eta^2d}{2} + 2\eta \times \\
&\sum_{i=1}^{d}|\nabla F({\color{black}\boldsymbol{w}}^{({\color{black}k})})_{i}|P\big({\color{black}\tilde{\boldsymbol{g}}}^{({\color{black}k})}_{i} = sign(\nabla F({\color{black}\boldsymbol{w}}^{({\color{black}k})}))_{i}\big).
\end{split}
\end{equation}
\end{proof}

\section{Proof of Lemma \ref{Lemma1}}\label{proofOfL1}
\begin{proof}
According to the constraint $1 - e^{-\frac{(2^{r_m}-1)N_{0}B_{m}}{P_m}} \leq p_{out,m}$, when $\max\{r_{m}^{(1)},r_{m}^{(3)}\} \leq r_{m} \leq r_{m}^{(2)}$, it can be obtained that
\begin{equation}
P_{m} \geq -\frac{N_{0}B_{m}(2^{r_{m}}-1)}{\ln(1-p_{out,m})}.
\end{equation}
Since the objective function $\frac{\alpha_m}{2}c_{m}D_{m}f_{m}^2 + \frac{P_{m}s_{m}}{r_{m}B_{m}}$ is an increasing function of $P_{m}$, we have
\begin{equation}
P_{m}^{*} = -\frac{N_{0}B_{m}(2^{r_{m}}-1)}{\ln(1-p_{out,m})}.
\end{equation}
According to the constraint $\frac{c_{m}D_{m}}{f_{m}} + \frac{s_m}{r_{m}B_{m}} \leq T_l$, we have
\begin{equation}
f_{m} \geq \frac{c_{m}D_{m}}{T_l-\frac{s_m}{r_{m}B_{m}}}.
\end{equation}
In addition, the objective function is an increasing function of $f_{m}$. Therefore,
\begin{equation}
f^{*}_{m} = \max\bigg\{\frac{c_{m}D_{m}}{T_l-\frac{s_m}{r_{m}B_{m}}}, f_{min,m}\bigg\}
\end{equation}
\end{proof}

\section{Proof of Lemma \ref{Lemma2}}\label{proofOfL2}
\begin{proof}
According to the constraint $\frac{\alpha_m}{2}c_{m}D_{m}f_{m}^2 + \frac{P_{m}s_{m}}{r_{m}B_{m}} \leq E_{m}$, we have
\begin{equation}
r_{m} \geq \frac{P_{m}s_{m}}{B_{m}(E_{m}-\frac{\alpha_m}{2}c_{m}D_{m}f_{m}^2)}.
\end{equation}
According to the constraint $\frac{c_{m}D_{m}}{f_{m}} + \frac{s_m}{r_{m}B_{m}} \leq T_l$, we have
\begin{equation}
r_{m} \geq \frac{s_{m}f_{m}}{B_{m}f_{m}T_{l}-B_{m}c_{m}D_{m}}.
\end{equation}
In addition, it can be shown that the objective function $\frac{M - 2\sum_{m=1}^{M}p_{out}(r_m)}{\sqrt{T_l}}$ is a decreasing function of $r_{m}$. Therefore,
\small
\begin{equation}
r^{*}_{m} = \max\bigg\{\frac{P_{m}s_{m}}{B_{m}(E_{m}-\frac{\alpha_m}{2}c_{m}D_{m}f_{m}^2)}, \frac{s_{m}f_{m}}{B_{m}f_{m}T_{l}-B_{m}c_{m}D_{m}}\bigg\}.
\end{equation}
\normalsize
\end{proof}
\color{black}
\section{Proof of Lemma \ref{LemmaHeterogeneous2}}\label{ProofofLemma32}
\noindent {\color{black}In the interest of space, we {\color{black}present} the proof of Lemma \ref{LemmaHeterogeneous2} when $M$ is odd. The corresponding result for even $M$ can be obtained following a similar strategy.}
\begin{proof}

Define a series of random variables $\{\hat{X}_{m,i}\}_{m=1}^{M}$ given by
\begin{equation}
\hat{X}_{m,i} =
\begin{cases}
1, \hfill ~~~\text{if $sign({\color{black}\hat{\boldsymbol{g}}_{m}^{({\color{black}k})}})_{i} \neq sign(\nabla F_{m}({\color{black}\boldsymbol{w}}^{({\color{black}k})}))_{i}$}, \\
0, \hfill ~~~\text{if $sign({\color{black}\hat{\boldsymbol{g}}_{m}^{({\color{black}k})}})_{i}  = sign(\nabla F_{m}({\color{black}\boldsymbol{w}}^{({\color{black}k})}))_{i}$}.
\end{cases}
\end{equation}
It can be verified that
\begin{equation}
\color{black}
\begin{split}
P(\hat{X}_{m,i} = 1) &= (1-p_{m}^{i})p_{out}(r_{m}) + p_{m}^{i}(1-p_{out}(r_m)) \\
&=\frac{1}{2} - b|\nabla F_{m}({\color{black}\boldsymbol{w}}^{({\color{black}k})})_i|.
\end{split}
\end{equation}


\begin{equation}
sign({\color{black}\hat{\boldsymbol{g}}_{m}^{({\color{black}k})}})_{i} =
\begin{cases}
\hfill 1, \hfill \text{with probability $\frac{1+b\nabla F_{m}({\color{black}\boldsymbol{w}}^{({\color{black}k})})_{i}}{2}$},\\
\hfill -1, \hfill \text{with probability $\frac{1-b\nabla F_{m}({\color{black}\boldsymbol{w}}^{({\color{black}k})})_{i}}{2}$},\\
\end{cases}
\end{equation}

Further define a series of random variables $\{\hat{Z}_{m,i}\}_{m=1}^{M}$ given by
\begin{equation}
\hat{Z}_{m,i} =
\begin{cases}
\hfill 1, \hfill &\text{if $sign({\color{black}\hat{\boldsymbol{g}}_{m}^{({\color{black}k})}})_{i} =1$},\\
\hfill 0, \hfill &\text{if $sign({\color{black}\hat{\boldsymbol{g}}_{m}^{({\color{black}k})}})_{i} = -1$.}
\end{cases}
\end{equation}

Let $\hat{Z}_{i} = \sum_{m=1}^{M}\hat{Z}_{m,i}$, then
\begin{equation}
\begin{split}
P\bigg(sign\bigg(\frac{1}{M}\sum_{m=1}^{M}sign({\color{black}\hat{\boldsymbol{g}}_{m}^{({\color{black}k})}})_{i}\bigg)=1\bigg) &= P\bigg(\hat{Z}_{i} \geq \frac{M}{2}\bigg)\\
&= \sum_{H = \lceil\frac{M+1}{2}\rceil}^{M}P(\hat{Z}_{i} = H).
\end{split}
\end{equation}

In addition, let $u_{m} = \nabla F_{m}({\color{black}\boldsymbol{w}}^{({\color{black}k})})_{i}$, we have
\begin{equation}\label{SPhatZ}
\begin{split}
&P(\hat{Z}_{i}=H) \\
&= \frac{\sum_{A \in F_H}\prod_{i \in A}(1+bu_{i})\prod_{j \in A^{c}}(1-bu_{j})}{2^{M}} \\
&= \frac{a_{M,H}b^{M} + a_{M-1,H}b^{M-1} + \cdots + a_{0,H}b^{0}}{2^{M}},
\end{split}
\end{equation}
in which $F_H$ is the set of all subsets of $H$ integers that can be selected from $\{1,2,3,...,M\}$; $a_{m,H}, \forall 0\leq m \leq M$ is some constant. It can be easily verified that $a_{0,H} = {M \choose H}$.

When $b$ is sufficiently small, $P(\hat{Z}_{i}=H)$ is dominated by the last two terms in (\ref{SPhatZ}). In particular, $\forall m$, we have
\begin{equation}
\begin{split}
&\sum_{A \in F_H}\prod_{i \in A}(1+bu_{i})\prod_{j \in A^{c}}(1-bu_{j})\\ &= (1+bu_{m})\sum_{A \in F_H, m\in A}\prod_{i \in A/\{m\}}(1+bu_{i})\prod_{j \in A^{c}}(1-bu_{j}) \\
&+ (1-bu_{m})\sum_{A \in F_H, m\notin A}\prod_{i \in A}(1+bu_{i})\prod_{j \in A^{c}/\{m\}}(1-bu_{j}).
\end{split}
\end{equation}
As a result, when $\lceil\frac{M+1}{2}\rceil \leq H \leq M-1$, the $u_{m}$ related term in $a_{1,H}$ is given by $\big[{M-1 \choose H-1} - {M-1 \choose H}\big]u_{m}$; when $H = M$, the $u_{m}$ related term in $a_{M-1,H}$ is given by $\big[{M-1 \choose H-1}\big]u_{m}$. By summing over $m$, we have
\begin{equation}
\begin{split}
  a_{1,H} = \bigg[{M-1 \choose H-1} - {M-1 \choose H}\bigg]\sum_{m=1}^{M}u_{m}, \\
\text{if}~~\bigg\lceil\frac{M+1}{2}\bigg\rceil \leq H \leq M-1,
\end{split}
\end{equation}
and
\begin{equation}
a_{1,H} = \bigg[{M-1 \choose H-1}\bigg]\sum_{m=1}^{M}u_{m},~~\text{if}~~H = M.
\end{equation}

By summing over $H$, we have
\begin{equation}
\sum_{H = \lceil\frac{M+1}{2}\rceil}^{M}a_{0,H} = \sum_{H = \lceil\frac{M+1}{2}\rceil}^{M}{M \choose H} = 2^{M-1},
\end{equation}
\begin{equation}
\sum_{H = \lceil\frac{M+1}{2}\rceil}^{M}a_{1,H} = {M-1 \choose \lceil\frac{M+1}{2}\rceil-1}\sum_{m=1}^{M}u_{m}
\end{equation}
As a result,
\begin{equation}
\begin{split}
&P\bigg(\hat{Z}_{i} \geq \frac{M}{2}\bigg) \\
&= \sum_{H = \lceil\frac{M+1}{2}\rceil}^{M}P(\hat{Z}_{i} = H) \\
&= \frac{2^{M-1} + {M-1 \choose \lceil\frac{M+1}{2}\rceil-1}\sum_{m=1}^{M}u_{m}b}{2^{M}} + O\bigg(\frac{b^2}{2^{M}}\bigg)\\
& = \frac{1}{2} + \frac{{M-1 \choose \lceil\frac{M+1}{2}\rceil-1}}{2^{M}}\sum_{m=1}^{M}u_{m}b + O\bigg(\frac{b^2}{2^{M}}\bigg),
\end{split}
\end{equation}
which completes the proof.
\end{proof}

\color{black}
\bibliography{Ref-FL}

\begin{thebibliography}{10}
\providecommand{\url}[1]{#1}
\csname url@samestyle\endcsname
\providecommand{\newblock}{\relax}
\providecommand{\bibinfo}[2]{#2}
\providecommand{\BIBentrySTDinterwordspacing}{\spaceskip=0pt\relax}
\providecommand{\BIBentryALTinterwordstretchfactor}{4}
\providecommand{\BIBentryALTinterwordspacing}{\spaceskip=\fontdimen2\font plus
\BIBentryALTinterwordstretchfactor\fontdimen3\font minus
  \fontdimen4\font\relax}
\providecommand{\BIBforeignlanguage}[2]{{%
\expandafter\ifx\csname l@#1\endcsname\relax
\typeout{** WARNING: IEEEtran.bst: No hyphenation pattern has been}%
\typeout{** loaded for the language `#1'. Using the pattern for}%
\typeout{** the default language instead.}%
\else
\language=\csname l@#1\endcsname
\fi
#2}}
\providecommand{\BIBdecl}{\relax}
\BIBdecl

\bibitem{konevcny2016federated}
J.~Kone{\v{c}}n{\`y}, H.~B. McMahan, D.~Ramage, and P.~Richt{\'a}rik,
  ``Federated optimization: Distributed machine learning for on-device
  intelligence,'' \emph{arXiv preprint arXiv:1610.02527}, 2016.

\bibitem{tran2019federated}
N.~H. {Tran}, W.~{Bao}, A.~{Zomaya}, M.~N.~H. {Nguyen}, and C.~S. {Hong},
  ``Federated learning over wireless networks: Optimization model design and
  analysis,'' in \emph{Proc. IEEE Conf. Comput. Commun. (IEEE INFOCOM)}, Paris,
  France, 2019, pp. 1387--1395.

\bibitem{dinh2020federated}
C.~T. Dinh, N.~H. Tran, M.~N. Nguyen, C.~S. Hong, W.~Bao, A.~Y. Zomaya, and
  V.~Gramoli, ``Federated learning over wireless networks: Convergence analysis
  and resource allocation,'' \emph{IEEE/ACM Trans. Netw.}, vol.~29, no.~1, pp.
  398--409, 2020.

\bibitem{yang2020delay}
Z.~Yang, M.~Chen, W.~Saad, C.~S. Hong, M.~Shikh-Bahaei, H.~V. Poor, and S.~Cui,
  ``Delay minimization for federated learning over wireless communication
  networks,'' \emph{arXiv preprint arXiv:2007.03462}, 2020.

\bibitem{zeng2020energy}
Q.~Zeng, Y.~Du, K.~Huang, and K.~K. Leung, ``Energy-efficient radio resource
  allocation for federated edge learning,'' in \emph{Proc. IEEE Int. Conf.
  Commun. Workshops (ICC Workshops)}, Dublin, Ireland, 2020, pp. 1--6.

\bibitem{yang2020energy}
Z.~Yang, M.~Chen, W.~Saad, C.~S. Hong, and M.~Shikh-Bahaei, ``Energy efficient
  federated learning over wireless communication networks,'' \emph{IEEE Trans.
  Wireless Commun.}, vol.~8, pp. 48\,088--48\,100, 2020.

\bibitem{shi2020device}
W.~Shi, S.~Zhou, and Z.~Niu, ``Device scheduling with fast convergence for
  wireless federated learning,'' in \emph{Proc. IEEE Int. Conf. Commun. (ICC)},
  Dublin, Ireland, 2020, pp. 1--6.

\bibitem{shi2020joint}
W.~Shi, S.~Zhou, Z.~Niu, M.~Jiang, and L.~Geng, ``Joint device scheduling and
  resource allocation for latency constrained wireless federated learning,''
  \emph{IEEE Trans. Wireless Commun.}, vol.~20, no.~1, pp. 453--467, 2021.

\bibitem{chen2020convergence}
M.~Chen, H.~V. Poor, W.~Saad, and S.~Cui, ``Convergence time optimization for
  federated learning over wireless networks,'' \emph{arXiv preprint
  arXiv:2001.07845}, 2020.

\bibitem{wadu2020federated}
M.~M. Wadu, S.~Samarakoon, and M.~Bennis, ``Federated learning under channel
  uncertainty: Joint client scheduling and resource allocation,'' \emph{arXiv
  preprint arXiv:2002.00802}, 2020.

\bibitem{vu2020cell}
T.~T. Vu, D.~T. Ngo, N.~H. Tran, H.~Q. Ngo, M.~N. Dao, and R.~H. Middleton,
  ``Cell-free massive mimo for wireless federated learning,'' \emph{IEEE Trans.
  Wireless Commun.}, vol.~19, no.~10, pp. 6377--6392, 2020.

\bibitem{ren2020accelerating}
J.~Ren, G.~Yu, and G.~Ding, ``Accelerating {DNN} training in wireless federated
  edge learning systems,'' \emph{IEEE J. Sel. Areas Commun.}, vol.~39, no.~1,
  pp. 219--232, 2021.

\bibitem{chen2020joint}
M.~Chen, Z.~Yang, W.~Saad, C.~Yin, H.~V. Poor, and S.~Cui, ``A joint learning
  and communications framework for federated learning over wireless networks,''
  \emph{IEEE Trans. Wireless Commun.}, vol.~20, no.~1, pp. 269--283, 2020.

\bibitem{amiri2020update}
M.~M. Amiri, D.~Gunduz, S.~R. Kulkarni, and H.~V. Poor, ``Update aware device
  scheduling for federated learning at the wireless edge,'' \emph{arXiv
  preprint arXiv:2001.10402}, 2020.

\bibitem{du2020high}
Y.~Du, S.~Yang, and K.~Huang, ``High-dimensional stochastic gradient
  quantization for communication-efficient edge learning,'' \emph{IEEE Trans.
  Signal Process.}, vol.~68, pp. 2128--2142, 2020.

\bibitem{zheng2020design}
S.~Zheng, C.~Shen, and X.~Chen, ``Design and analysis of uplink and downlink
  communications for federated learning,'' \emph{IEEE J. Sel. Areas Commun.},
  vol.~39, no.~7, pp. 2150--2167, 2021.

\bibitem{chang2020communication}
W.-T. Chang and R.~Tandon, ``Communication efficient federated learning over
  multiple access channels,'' \emph{arXiv preprint arXiv:2001.08737}, 2020.

\bibitem{zhu2020broad}
G.~Zhu, Y.~Wang, and K.~Huang, ``Broadband analog aggregation for low-latency
  federated edge learning,'' \emph{IEEE Trans. Wireless Commun.}, vol.~19,
  no.~1, pp. 491--506, 2020.

\bibitem{zhu2020toward}
G.~Zhu, D.~Liu, Y.~Du, C.~You, J.~Zhang, and K.~Huang, ``Toward an intelligent
  edge: Wireless communication meets machine learning,'' \emph{IEEE Commun.
  Mag.}, vol.~58, no.~1, pp. 19--25, 2020.

\bibitem{yang2020federated}
K.~Yang, T.~Jiang, Y.~Shi, and Z.~Ding, ``Federated learning via over-the-air
  computation,'' \emph{IEEE Trans. Wireless Commun.}, vol.~19, no.~3, pp.
  2022--2035, 2020.

\bibitem{amiri2020machine}
M.~M. Amiri and D.~G{\"u}nd{\"u}z, ``Machine learning at the wireless edge:
  Distributed stochastic gradient descent over-the-air,'' \emph{IEEE Trans.
  Signal Process.}, vol.~68, pp. 2155--2169, 2020.

\bibitem{amiri2020federated}
------, ``Federated learning over wireless fading channels,'' \emph{IEEE Trans.
  Wireless Commun.}, vol.~19, no.~5, pp. 3546--3557, 2020.

\bibitem{zhu2020one}
G.~Zhu, Y.~Du, D.~Gunduz, and K.~Huang, ``One-bit over-the-air aggregation for
  communication-efficient federated edge learning: Design and convergence
  analysis,'' \emph{arXiv preprint arXiv:2001.05713}, 2020.

\bibitem{hosseinalipour2020federated}
S.~Hosseinalipour, C.~G. Brinton, V.~Aggarwal, H.~Dai, and M.~Chiang, ``From
  federated learning to fog learning: Towards large-scale distributed machine
  learning in heterogeneous wireless networks,'' \emph{IEEE Commun. Mag.}, to
  appear.

\bibitem{hosseinalipour2020multi}
S.~Hosseinalipour, S.~S. Azam, C.~G. Brinton, N.~Michelusi, V.~Aggarwal, D.~J.
  Love, and H.~Dai, ``Multi-stage hybrid federated learning over large-scale
  wireless fog networks,'' \emph{arXiv preprint arXiv:2007.09511}, 2020.

\bibitem{bernstein2018signsgd1}
J.~Bernstein, Y.-X. Wang, K.~Azizzadenesheli, and A.~Anandkumar, ``sign{SGD}:
  Compressed optimisation for non-convex problems,'' in \emph{Proc. Int. Conf.
  Mach. Learn. (ICML)}, Stockholm, Sweden, 2018, pp. 560--569.

\bibitem{mcmahan2017communication}
B.~McMahan, E.~Moore, D.~Ramage, S.~Hampson, and B.~A. y~Arcas,
  ``Communication-efficient learning of deep networks from decentralized
  data,'' in \emph{Proc. Int. Conf. Artif. Intell. Stat.}, Ft. Lauderdale, FL,
  USA, 2017, pp. 1273--1282.

\bibitem{alistarh2017qsgd}
D.~Alistarh, D.~Grubic, J.~Li, R.~Tomioka, and M.~Vojnovic, ``{QSGD}:
  Communication-efficient {SGD} via gradient quantization and encoding,'' in
  \emph{Proc. Adv. Neural Inf. Process. Syst.}, Long Beach, CA, USA, 2017, pp.
  1709--1720.

\bibitem{wu2018error}
J.~Wu, W.~Huang, J.~Huang, and T.~Zhang, ``Error compensated quantized {SGD}
  and its applications to large-scale distributed optimization,'' in
  \emph{Proc. Int. Conf. Mach. Learn. (ICML)}, Stockholm, SWEDEN, 2018, pp.
  5325--5333.

\bibitem{wen2017terngrad}
W.~Wen, C.~Xu, F.~Yan, C.~Wu, Y.~Wang, Y.~Chen, and H.~Li, ``Terngrad: Ternary
  gradients to reduce communication in distributed deep learning,'' in
  \emph{Proc. Adv. Neural Inf. Process. Syst.}, Long Beach, CA, USA, 2017, pp.
  1509--1519.

\bibitem{agarwal2018cpsgd}
N.~Agarwal, A.~T. Suresh, F.~X.~X. Yu, S.~Kumar, and B.~McMahan, ``{cpSGD}:
  Communication-efficient and differentially-private distributed {SGD},'' in
  \emph{Proc. Adv. Neural Inf. Process. Syst.}, Montréal, CANADA, 2018, pp.
  7564--7575.

\bibitem{sattler2019sparse}
F.~Sattler, S.~Wiedemann, K.-R. M{\"u}ller, and W.~Samek, ``Sparse binary
  compression: Towards distributed deep learning with minimal communication,''
  in \emph{Proc. IEEE Int. Joint Conf. Neural Netw. (IJCNN)}, Budapest,
  Hungary, 2019, pp. 1--8.

\bibitem{sattler2019robust}
------, ``Robust and communication-efficient federated learning from non-iid
  data,'' \emph{IEEE Trans. Neural Netw. Learn. Syst.}, vol.~31, no.~9, pp.
  3400--3413, 2020.

\bibitem{wang2018atomo}
H.~Wang, S.~Sievert, S.~Liu, Z.~Charles, D.~Papailiopoulos, and S.~Wright,
  ``Atomo: Communication-efficient learning via atomic sparsification,'' in
  \emph{Proc. Adv. Neural Inf. Process. Syst.}, Montréal, CANADA, 2018, pp.
  9850--9861.

\bibitem{caldas2018expanding}
S.~Caldas, J.~Kone{\v{c}}ny, H.~B. McMahan, and A.~Talwalkar, ``Expanding the
  reach of federated learning by reducing client resource requirements,''
  \emph{arXiv preprint arXiv:1812.07210}, 2018.

\bibitem{burd1996processor}
T.~D. Burd and R.~W. Brodersen, ``Processor design for portable systems,''
  \emph{J. VLSI Signal Process. Syst.}, vol.~13, no. 2-3, pp. 203--221, 1996.

\bibitem{goldsmith2005wireless}
A.~Goldsmith, \emph{Wireless communications}.\hskip 1em plus 0.5em minus
  0.4em\relax Cambridge university press, 2005.

\bibitem{bonawitz2019towards}
K.~Bonawitz, H.~Eichner, W.~Grieskamp, D.~Huba, A.~Ingerman, V.~Ivanov,
  C.~Kiddon, J.~Konecny, S.~Mazzocchi, H.~B. McMahan \emph{et~al.}, ``Towards
  federated learning at scale: System design,'' \emph{arXiv preprint
  arXiv:1902.01046}, 2019.

\bibitem{tao1997convex}
P.~D. Tao and L.~T.~H. An, ``Convex analysis approach to {DC} programming:
  theory, algorithms and applications,'' \emph{Acta Math. Vietnamica}, vol.~22,
  no.~1, pp. 289--355, 1997.

\bibitem{boyd2003subgradient}
S.~Boyd, L.~Xiao, and A.~Mutapcic, ``Subgradient methods,'' \emph{lecture notes
  of EE392o, Stanford University, Autumn Quarter 2003-2004}.

\bibitem{chen2019distributed}
X.~Chen, T.~Chen, H.~Sun, Z.~S. Wu, and M.~Hong, ``Distributed training with
  heterogeneous data: Bridging median and mean based algorithms,'' \emph{arXiv
  preprint arXiv:1906.01736}, 2019.

\bibitem{jin2020stochastic}
R.~Jin, Y.~Huang, X.~He, H.~Dai, and T.~Wu, ``Stochastic-{Sign} {SGD} for
  federated learning with theoretical guarantees,'' \emph{arXiv preprint
  arXiv:2002.10940}, 2020.

\bibitem{lin2020ensemble}
T.~Lin, L.~Kong, S.~U. Stich, and M.~Jaggi, ``Ensemble distillation for robust
  model fusion in federated learning,'' in \emph{Proc. Neural Inf. Process.
  Syst. (NeurIPS)}, 2020.

\bibitem{chen2017distributed}
Y.~Chen, L.~Su, and J.~Xu, ``Distributed statistical machine learning in
  adversarial settings: Byzantine gradient descent,'' \emph{ACM Meas. Anal.
  Comput. Syst.}, vol.~1, no.~2, pp. 1--25, 2017.

\end{thebibliography}
\bibliographystyle{IEEEtran}
\end{document}